\documentclass{article} 
\usepackage[]{iclr2027_conference,times}

\usepackage{amsmath,amsfonts,bm}

\def\eqref#1{equation~\ref{#1}}

\def\1{\bm{1}}

\DeclareMathAlphabet{\mathsfit}{\encodingdefault}{\sfdefault}{m}{sl}
\SetMathAlphabet{\mathsfit}{bold}{\encodingdefault}{\sfdefault}{bx}{n}

\newcommand{\R}{\mathbb{R}}

\usepackage[utf8]{inputenc}
\usepackage[T1]{fontenc}
\usepackage{hyperref}
\usepackage{url}
\usepackage{booktabs}
\usepackage{amsfonts}
\usepackage{amsmath}
\usepackage{amssymb}
\usepackage{nicefrac}
\usepackage{microtype}
\usepackage{xcolor}
\usepackage{graphicx}
\usepackage{algorithm}
\usepackage{algorithmic}
\usepackage{amsthm}
\usepackage{tikz}
\usepackage{multirow}
\usetikzlibrary{arrows.meta, positioning, calc, decorations.pathreplacing, fit}
\usepackage[table]{xcolor}

\newtheorem{proposition}{Proposition}
\newtheorem{definition}{Definition}
\newtheorem{remark}{Remark}

\newcommand{\btheta}{\boldsymbol{\theta}}

\title{GPart: End-to-End Isometric Fine-Tuning via Global Parameter Partitioning}

\author{Paolo Mandica\textsuperscript{$\dagger$}$^1$\thanks{Corresponding author: \texttt{p.mandica@samsung.com}} \quad Micha{\l} Brzozowski\textsuperscript{$\dagger$}$^1$ \quad Zuzanna Dubanowska$^1$ \quad Neo C. Chung$^{1,2}$ \\
$^1$Samsung AI Center, Warsaw, Poland \quad $^2$University of Warsaw, Warsaw, Poland
}

\iclrpreprint

\begin{document}

\maketitle

\begingroup\renewcommand\thefootnote{$\dagger$}
\footnotetext{Equal contribution}
\endgroup

\begin{abstract}
Low-rank adaptation (LoRA) has become a dominant paradigm for parameter-efficient fine-tuning (PEFT) of large-scale deep learning models. However, its bilinear parameterization induces a parameter-dependent geometry: the mapping from trainable parameters to weight updates is not generally distance-preserving. Related methods that project a low-dimensional vector into LoRA's parameter space, such as Uni-LoRA, improve parameter efficiency, but the subsequent bilinear map breaks end-to-end isometry. We propose GPart (Global Partition fine-tuning), a highly parameter-efficient fine-tuning method that maps a $d$-dimensional trainable vector directly into the full weight space through a sparse, isometric partition matrix. GPart retains a fixed global parameter-sharing prior while removing the additional low-rank reconstruction used by LoRA-based methods. This yields a simple parameterization with a single main hyperparameter ($d$), exact end-to-end isometry, and a minimal checkpoint representation consisting of the trainable vector and a random seed. GPart builds on the premise of effective fine-tuning within random low-dimensional subspaces of the full weight space \emph{without requiring a low-rank matrix factorization}. Across natural language understanding, computer vision, and mathematical reasoning benchmarks, GPart matches or improves over existing PEFT methods at ultra-low parameter budgets. Beyond offering mathematical tractability and memory efficiency, the direct linear parameterization of GPart streamlines model selection and paves the way for compact adapter composition. Overall, GPart provides an elegant and competitive alternative for fine-tuning under small parameter budgets, with a fixed and predictable geometry between trainable coordinates and weight-space updates.
\end{abstract}

\section{Introduction}
\label{sec:intro}

Modern deep learning models for language, vision, and other modalities are often impractical to fine-tune conventionally due to a massive computational cost and time. Nonetheless, fine-tuning on downstream tasks remains an effective method that may outperform in-context learning \cite{bucher2024fine, walsh2025fine}. To overcome the challenge of updating the full model weights, parameter-efficient fine-tuning (PEFT) restricts updates to a low-dimensional trainable subspace. Among PEFT methods, LoRA~\citep{hu2022lora} is the most widely used: for a weight matrix $W \in \mathbb{R}^{m \times n}$, it parameterizes the update as
\begin{equation}
\Delta W = BA, \qquad B \in \mathbb{R}^{m \times r}, \; A \in \mathbb{R}^{r \times n},
\label{eq:lora}
\end{equation}
so that only $r(m+n)$ parameters are trained per layer.

LoRA is computationally efficient and effective, but its low-rank bilinear parameterization imposes additional, arbitrary structure on the trainable subspace. In particular, the map from trainable parameters $(A,B)$ to the weight update $BA$ is not an isometry, so Euclidean distances in the parameter space are not preserved in the weight space. As a result, the geometry seen by the optimizer in the trainable coordinates need not align with the geometry of the induced weight updates.

\begin{figure}[ht]
    \centering
    \includegraphics[width=\linewidth]{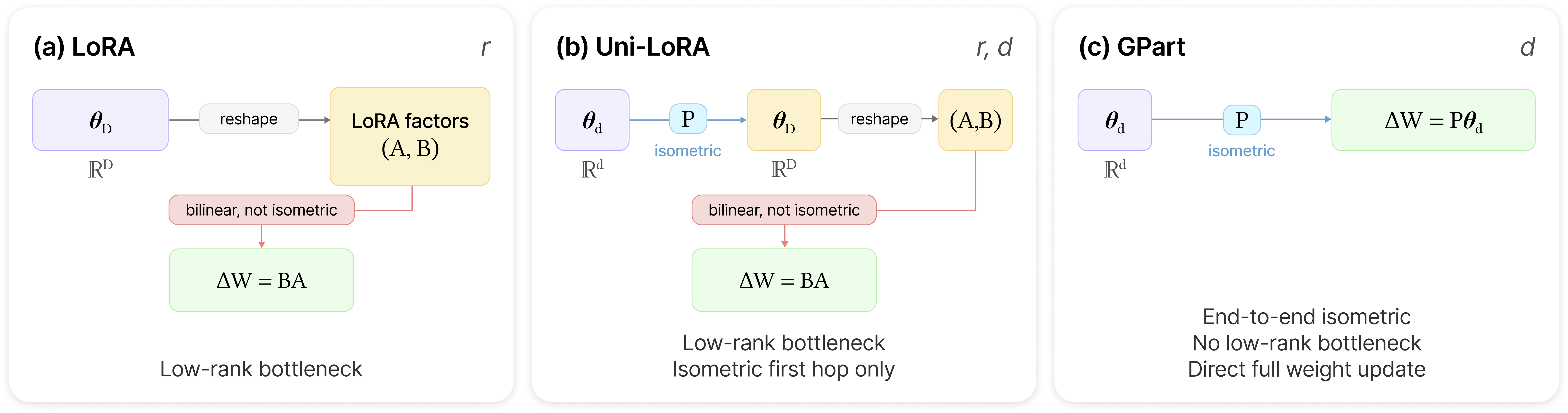}
    \caption{Comparison of PEFT parameterizations. LoRA and Uni-LoRA construct weight updates through the bilinear map $\Delta W = BA$, which breaks end-to-end distance preservation. GPart instead projects the trainable vector directly into full weight space with a seed-generated partition matrix $P$, yielding a one-step isometric parameterization with $d$ as the only parameter-budget hyperparameter.}
    \label{fig:diagram_1}
\end{figure}

Recent work has pushed PEFT to even smaller trainable parameter budgets. VeRA~\citep{kopiczko2024vera} freezes random matrices and trains only scaling vectors. Uni-LoRA~\citep{li2025unilora} further compresses the trainable space by optimizing a single $d$-dimensional vector $\btheta_d$, which is mapped into the LoRA parameter space $\mathbb{R}^D$ through a random partition matrix $P \in \mathbb{R}^{D \times d}$ satisfying $P^\top P = I_d$. Despite this projection into LoRA parameter space is isometric, Uni-LoRA still requires the LoRA bilinear map, so the overall map from $\btheta_d$ to weight space is not isometric (Figure~\ref{fig:diagram_1}b).

We propose \textbf{GPart} (Global Partition fine-tuning), which removes the intermediate low-rank parameterization entirely. Instead of projecting into LoRA space, GPart maps the trainable vector directly into the full weight space:
\begin{equation}
W = W_0 + P\btheta_d, \qquad P \in \mathbb{R}^{N \times d}, \qquad P^\top P = I_d.
\label{eq:gpart}
\end{equation}
This yields a single linear map $\mathbb{R}^d \xrightarrow{P,\ \text{isometry}} \mathbb{R}^N$, so the trainable coordinates are connected to the weight update through an end-to-end isometric parameterization of the optimized subspace. In a comprehensive evaluation ranging from natural language understanding to computer vision, GPart demonstrates superior or comparable performance at ultra-low parameter budgets.

Furthermore, mathematical tractability of GPart simplifies utilization of adapters downstream. The direct linear parameterization supports compact adapter composition.
For adapters sharing the same projection $P$, additive merging in the
trainable coordinates exactly reproduces additive merging of their
weight updates:
$P(\btheta_d^{(1)}+\btheta_d^{(2)})
= P\btheta_d^{(1)}+P\btheta_d^{(2)}$,
as analyzed in Appendix~H. This parameterization also simplifies model selection. In LoRA, the rank $r$ controls both the size and the structure of the low-rank update. In Uni-LoRA, the trainable budget is controlled by $d$, but $r$ must still be chosen because the method retains the LoRA factorization. GPart removes this extra choice: $d$, the number of partition groups, is the only hyperparameter controlling subspace size (Figure~\ref{fig:diagram_1}).

Conceptually, GPart reconnects PEFT to intrinsic-dimension results of \citet{aghajanyan2021}, which showed that strong fine-tuning performance can emerge from optimization in random low-dimensional subspaces of the full parameter space. Unlike their structured Fastfood projection, GPart uses a sparse, seed-reconstructible partition that is an exact isometry for every realized projection. Importantly, Fastfood requires $\sim60\%$ more training time than GPart at the training protocol. GPart therefore follows the intrinsic-dimension perspective directly while retaining the compact representation pioneered by VeRA~\citep{kopiczko2024vera} and Uni-LoRA~\citep{li2025unilora}.

Our contributions are as follows:
\begin{itemize}
    \item We introduce GPart, a PEFT method that maps a $d$-dimensional trainable vector directly into the full weight space via a random partition matrix, eliminating the intermediate low-rank factorization used in LoRA-based approaches.
    \item We show theoretically that GPart preserves Euclidean geometry within the trainable subspace, whereas the LoRA bilinear map does not. This direct linear parameterization results in straightforward hyper-parameter selection and compact adapter composition. 
    \item We empirically show that GPart is a competitive alternative to Uni-LoRA and other PEFT methods across encoder, decoder, and vision settings. At matched parameter budgets, GPart matches or improves over Uni-LoRA on average, with the clearest gains on natural language understanding and computer vision tasks.
\end{itemize}
\section{Related Work}
\label{sec:related}

\paragraph{Low-rank adaptation.}
LoRA~\citep{hu2022lora} parameterizes weight updates as a low-rank product, \(\Delta W = BA\), and has become a de facto standard to parameter-efficient fine-tuning (PEFT). A number of follow-up methods modify this parameterization. For example, DoRA~\citep{liu2024dora} separates magnitude and direction, while other variants such as AdaLoRA~\citep{zhang2023adalora} adapt the rank allocation across layers. Despite these differences, these methods retain the low-rank bilinear structure that maps trainable parameters to weight updates through \(BA\).

\paragraph{Hyperefficiency in PEFT.}
Following the introduction of LoRA, related PEFT methods are introduced to further reduce the number of trainable parameters. BitFit~\citep{zaken2022bitfit} updates only bias parameters, showing that strong adaptation can be achieved with extremely small trainable subsets. VeRA~\citep{kopiczko2024vera} freezes random matrices and learns only scaling vectors, enabling compact storage through the learned parameters and a random seed. FourierFT~\citep{gao2024fourierft} reparameterizes weight updates in the frequency domain, learning a small number of Fourier coefficients per layer to reconstruct the full update via the inverse discrete Fourier transform. Uni-LoRA~\citep{li2025unilora} trains a single low-dimensional vector and projects it into the LoRA parameter space using a random partition matrix. GPart is closest in spirit to VeRA and Uni-LoRA in its use of seed-generated random projections, but differs in that it maps directly into the full weight space rather than into an intermediate LoRA parameterization.

\paragraph{Intrinsic dimensionality and random subspaces.}
Our work is also closely related to intrinsic-dimension approaches to fine-tuning. \citet{aghajanyan2021} showed that effective fine-tuning can often be achieved by optimizing within a random low-dimensional subspace of the full parameter space. Their method projects a \(d\)-dimensional trainable vector into \(\mathbb{R}^N\) using the Fastfood transform~\citep{le2013fastfood}, thereby directly parameterizing updates in the ambient weight space. GPart follows the same high-level random-subspace perspective, but uses a sparse, seed-reconstructible partition projection that is an exact isometry for each realized map, rather than a structured dense transform.

\paragraph{Hash-based parameter sharing.}
GPart is also related to earlier work on parameter sharing through hashing. HashedNet~\citep{chen2015hashednet} compresses neural networks by assigning weights to shared hash buckets, so that multiple weights reuse the same learned parameter. This mechanism is algebraically similar to the partition matrix used in GPart, where each parameter index is assigned to one of \(d\) groups. The setting, however, is different: HashedNet was introduced for model compression and training compact models, whereas GPart uses the same type of sharing structure to parameterize fine-tuning updates around a pretrained model.

\section{Method}
\label{sec:method}

\begin{figure}[t]
    \centering
    \includegraphics[width=\linewidth]{figures/Diagram_2.png}
    \caption{\textbf{Overview of GPart.} A $d$-dimensional trainable vector $\btheta_d$ is broadcast into the full weight space via a random partition generated from a seed $s$. Each model parameter $w_i$ is assigned to a group $g(i) \in \{1, \ldots, d\}$ and updated as $\Delta w_i = \theta_{g(i)} / \sqrt{n_{g(i)}}$, preserving isometry. A minimal GPart checkpoint consists of only the $d$ trainable values and the partition seed.}
    \label{fig:diagram_2}
\end{figure}

We begin by describing GPart in the vectorized weight space of the pretrained model and comparing it with the LoRA-based factor space; we then formalize the construction of the partition matrix and the resulting forward and backward passes. Figure~\ref{fig:diagram_2} provides an overview of the full pipeline.

\subsection{Vectorized weight space}

Consider a pretrained model with $L$ adapted weight matrices $W^{(1)}, W^{(2)}, \ldots, W^{(L)}$, possibly of different shapes. We flatten each matrix and concatenate the results into a single parameter vector
\begin{equation}
w_0 = \mathrm{Concat}\!\Big(\mathrm{vec}\big(W^{(1)}\big),\; \mathrm{vec}\big(W^{(2)}\big),\; \ldots,\; \mathrm{vec}\big(W^{(L)}\big)\Big) \in \mathbb{R}^N,
\label{eq:vectorize}
\end{equation}
where $N = \sum_{\ell=1}^L m_\ell n_\ell$ is the total number of adapted parameters.

For comparison, Uni-LoRA parameterizes updates in the LoRA factor space rather than in $\mathbb{R}^N$. For each adapted layer $\ell$, LoRA introduces factors $B^{(\ell)} \in \mathbb{R}^{m_\ell \times r}$ and $A^{(\ell)} \in \mathbb{R}^{r \times n_\ell}$. Flattening and concatenating these variables across layers gives
\begin{equation}
\btheta_D = \mathrm{Concat}\!\Big(\mathrm{vec}\big(B^{(1)}\big),\; \mathrm{vec}\big(A^{(1)}\big),\; \ldots,\; \mathrm{vec}\big(B^{(L)}\big),\; \mathrm{vec}\big(A^{(L)}\big)\Big) \in \mathbb{R}^D,
\label{eq:unilora-concat}
\end{equation}
with $D = \sum_{\ell=1}^L r(m_\ell + n_\ell)$ typically satisfying $D \ll N$.



Unlike Uni-LoRA, GPart operates directly in the model weight space $\mathbb{R}^N$. It therefore requires neither low-rank factors nor a bilinear map from factor space back to model space.

\subsection{Partition matrix}

Let $d \ll N$ denote the dimension of the trainable subspace, where $N$ is the total number of adapted parameters. GPart defines a sparse matrix $P \in \mathbb{R}^{N \times d}$ by first flattening these adapted parameters into a global vector of length $N$, applying a seed-dependent pseudorandom permutation, and then splitting the permuted sequence into $d$ disjoint groups. This yields a global assignment map $g : \{1,\ldots,N\} \to \{1,\ldots,d\}$,such that each parameter belongs to exactly one group and every group is nonempty.

For each group $j$, define $n_j = \big|\{i : g(i) = j\}\big|$.
The matrix $P$ is then given entrywise by
\begin{equation}
P_{ij} =
\begin{cases}
\frac{1}{\sqrt{n_j}}, & \text{if } g(i)=j,\\
0, & \text{otherwise.}
\end{cases}
\label{eq:partition-matrix}
\end{equation}

By construction, each row of $P$ contains exactly one nonzero entry, while distinct columns have disjoint supports. Since every group is nonempty, each column has unit norm, and therefore
\begin{equation}
P^\top P = I_d.
\label{eq:ptp}
\end{equation}
Hence, $P$ is an isometric embedding from $\mathbb{R}^d$ into $\mathbb{R}^N$.

The assignment is global: parameters are grouped across the entire model rather than separately within each layer. In practice, GPart is implemented through the seed-defined assignment $g(\cdot)$ and the group sizes $\{n_j\}_{j=1}^d$, without explicitly materializing $P$. The assignments may either be materialized as integer indices or generated statelessly from the seed and parameter positions, avoiding an $O(N)$ assignment buffer. Figure~\ref{fig:diagram_2} illustrates both the partition construction and the induced broadcast update.

\subsection{Forward and backward passes}

Given the pretrained weight vector $w_0 \in \mathbb{R}^N$ and trainable parameters $\btheta_d \in \mathbb{R}^d$, GPart defines the adapted weights as
\begin{equation}
w = w_0 + \Delta w, \qquad \Delta w = P\btheta_d.
\label{eq:forward}
\end{equation}
Equivalently, each parameter $i$ receives the update
\begin{equation}
\Delta w_i = \frac{\theta_{g(i)}}{\sqrt{n_{g(i)}}}.
\label{eq:coordinate-update}
\end{equation}
Thus, all parameters assigned to the same group share one trainable value, normalized by the group size. Computing $P\btheta_d$ requires only a single pass over the $N$ parameters and no explicit matrix multiplication.

For a loss $\mathcal{L}(w)$, the gradient with respect to $\btheta_d$ is
$
\nabla_{\btheta_d} \mathcal{L} = P^\top \nabla_w \mathcal{L}.
\label{eq:backward}
$
The backward pass therefore reduces to accumulating normalized gradient sums within each group, which again requires $O(N)$ work.

We initialize the trainable vector $\btheta_d$ at zero. Since GPart is linear in $\btheta_d$, this gives $\Delta w = P\btheta_d = 0$ at initialization, so optimization starts exactly from the pretrained model $w_0$. Unlike bilinear parameterizations such as LoRA, no symmetry-breaking random initialization is required: from Equation~\ref{eq:backward}, the gradient with respect to $\btheta_d$ is generally nonzero even when $\btheta_d = 0$.

\subsection{Storage and hyperparameter}

A minimal GPart checkpoint is fully specified by the seed $s$, which regenerates the partition, and the trainable vector $\btheta_d \in \mathbb{R}^d$, requiring only $d+1$ stored values. This matches the minimal checkpoint representation of Uni-LoRA while requiring fewer trainable adapter parameters than LoRA ($O(D)$) and full fine-tuning ($O(N)$). At runtime, a materialized implementation additionally requires group-assignment indices and normalization scales; these assignment indices can instead be generated statelessly from the seed and parameter positions, eliminating the $O(N)$ index buffer. We discuss the resulting runtime memory and computational trade-offs in Section~\ref{sec:experiments}.

GPart is controlled by a single hyperparameter, the subspace dimension $d$. Setting $d=1$ collapses all parameters to one shared scalar, while $d=N$ recovers full fine-tuning; intermediate values interpolate between these extremes. Unlike LoRA-based methods, GPart does not require choosing $r$, and uses $d$ alone to control the trade-off between parameter efficiency and expressiveness.
\section{Theoretical Analysis}
\label{sec:theory}

We analyze the geometry induced by GPart and contrast it with that of LoRA-based parameterizations. Our central observation is that GPart defines a linear isometric embedding from the trainable space into the full weight space, whereas LoRA-based methods map trainable parameters to weight updates through a bilinear transformation whose local geometry depends on the current parameter values. As a result, GPart provides fixed orthonormal coordinates for optimization within its trainable subspace, whereas the geometry induced by LoRA-based parameterizations varies throughout training.

\subsection{End-to-end isometry of GPart}

We begin with the basic geometric property of GPart.

\begin{proposition}[GPart isometry]
\label{prop:isometry}
Let $P \in \mathbb{R}^{N \times d}$ be the partition matrix defined in Section~\ref{sec:method}, satisfying $P^\top P = I_d$. Then for any $\btheta, \btheta' \in \mathbb{R}^d$,
\begin{equation}
\|P\btheta - P\btheta'\|_2 = \|\btheta - \btheta'\|_2.
\label{eq:gpart-isometry}
\end{equation}
\end{proposition}

\begin{proof}
By linearity,
$$
\|P\btheta - P\btheta'\|_2^2
=
\|P(\btheta-\btheta')\|_2^2
=
(\btheta-\btheta')^\top P^\top P (\btheta-\btheta')
=
(\btheta-\btheta')^\top (\btheta-\btheta')
=
\|\btheta-\btheta'\|_2^2.
$$
Taking square roots gives the result.
\end{proof}

Proposition~\ref{prop:isometry} shows that GPart preserves Euclidean geometry exactly within the trainable subspace. Equivalently, the map $\btheta \mapsto P\btheta$ preserves norms, distances, and inner products on $\mathbb{R}^d$. Thus, optimization in $\btheta$-space is exactly optimization over the subspace $\operatorname{image}(P) \subset \mathbb{R}^N$ expressed in orthonormal coordinates.





\subsection{LoRA induces a parameter-dependent metric}

We now contrast this with the LoRA parameterization. For a single layer, LoRA represents the weight update as
\begin{equation}
\phi(A,B) = BA,
\qquad
B \in \mathbb{R}^{m \times r}, \quad A \in \mathbb{R}^{r \times n}.
\label{eq:lora-map}
\end{equation}
Unlike GPart, this map is bilinear rather than linear. Its local behavior is therefore governed by a Jacobian that depends on the current values of $A$ and $B$.

Using $\operatorname{vec}(XYZ) = (Z^\top \otimes X)\operatorname{vec}(Y)$, we obtain
$\operatorname{vec}(BA) = (I_n \otimes B)\operatorname{vec}(A)$ and $\operatorname{vec}(BA) = (A^\top \otimes I_m)\operatorname{vec}(B)$.
Hence,
\begin{equation}
\frac{\partial\, \operatorname{vec}(BA)}{\partial\, \operatorname{vec}(A)}
=
I_n \otimes B,
\qquad
\frac{\partial\, \operatorname{vec}(BA)}{\partial\, \operatorname{vec}(B)}
=
A^\top \otimes I_m.
\label{eq:lora-jacobian}
\end{equation}

These Jacobian blocks depend explicitly on $A$ and $B$. Consequently, the metric induced on the trainable coordinates by the map $(A,B)\mapsto BA$ varies with the current point in parameter space. In particular, there is no fixed orthonormal coordinate system in $(A,B)$-space whose Euclidean geometry is preserved by the LoRA map throughout training. Equal-norm perturbations in trainable coordinates can therefore produce different update magnitudes depending on the values of $A$ and $B$.




\subsection{Connection to intrinsic dimensionality}

The motivation for GPart is closely related to the intrinsic-dimensionality view of fine-tuning introduced by \citet{aghajanyan2021}, where strong adaptation can emerge from optimization in random low-dimensional subspaces of the full parameter space. GPart follows this viewpoint by optimizing directly in a random $d$-dimensional subspace of $\mathbb{R}^N$. Its partition matrix satisfies $P^\top P = I_d$ for every realization, giving exact norm and inner-product preservation. This contrasts with structured random projections such as Fastfood~\citep{le2013fastfood}, which are isotropic in expectation but are not generally exact isometries for each realized projection. The key distinction from Uni-LoRA is the location of the random subspace: GPart operates directly in full weight space, whereas Uni-LoRA projects into LoRA parameter space before applying the bilinear map to weight space. Thus, the former preserves the geometry of the restricted problem exactly, whereas the latter does not.
\section{Experiments}
\label{sec:experiments}

We evaluate GPart across encoder-only and decoder-only models on three benchmark families: natural language understanding, mathematical reasoning, and computer vision. We compare against both non-PEFT baselines—Linear Probing (LP), which updates only the task-specific head, and Full Fine-tuning (FF), which updates all model parameters—and standard PEFT baselines, including LoRA~\citep{hu2022lora}, BitFit~\citep{zaken2022bitfit}, VeRA~\citep{kopiczko2024vera}, FourierFT~\citep{gao2024fourierft}, and Uni-LoRA~\citep{li2025unilora}, where applicable. The reported trainable-parameter counts exclude the task-specific head and include only backbone parameters for FF and adapter parameters for PEFT methods. Detailed training hyperparameters for each setup are provided in Appendix~\ref{app:implementation-details}.

\subsection{Natural language understanding}

\begin{table}[t]
\caption{GLUE validation results for RoBERTa-base and RoBERTa-large. Bold indicates the best result among PEFT methods, underlined entries indicate the second-best result, and highlighted entries mark cases where GPart outperforms Uni-LoRA.}
\label{tab:roberta-combined}
\centering
\small
\resizebox{.9\textwidth}{!}{%
\begin{tabular}{llcccccccc}
\toprule
Model & Method & \# Params & CoLA & SST-2 & MRPC & STS-B & QNLI & RTE & Avg. \\
\midrule
\multirow{7}{*}{\rotatebox[origin=c]{90}{RoBERTa-Base}}
& LP & 0
& $43.1_{\pm 2.6}$ & $84.3_{\pm 0.5}$ & $72.5_{\pm 0.5}$
& $66.3_{\pm 0.5}$ & $70.6_{\pm 0.1}$ & $59.2_{\pm 1.2}$ & 66.0 \\

& FF & 124M
& $60.6_{\pm 2.1}$ & $94.3_{\pm 0.5}$ & $87.8_{\pm 0.9}$
& $90.5_{\pm 0.1}$ & $92.2_{\pm 0.3}$ & $77.3_{\pm 0.5}$ & 83.8 \\
\cmidrule(lr){2-10}

& LoRA ($r{=}8$) & 294K
& $60.5_{\pm 1.2}$
& $94.0_{\pm 0.6}$
& $87.8_{\pm 0.5}$
& $\textbf{90.8}_{\pm 0.2}$
& $\textbf{92.8}_{\pm 0.1}$
& $75.4_{\pm 1.0}$
& \underline{83.5} \\

& BitFit & 102K
& $58.5_{\pm 1.0}$
& $93.6_{\pm 0.2}$
& $\underline{88.0}_{\pm 1.2}$
& $90.2_{\pm 0.04}$
& $91.8_{\pm 0.1}$
& $\underline{78.7}_{\pm 6.1}$
& \underline{83.5} \\

& VeRA ($r{=}1024$) & 43K
& $\textbf{60.8}_{\pm 0.7}$
& $94.0_{\pm 0.2}$
& $87.5_{\pm 0.2}$
& $90.2_{\pm 0.1}$
& $91.8_{\pm 0.1}$
& $76.2_{\pm 2.9}$
& 83.4 \\

& Uni-LoRA & 23K
& $58.1_{\pm 0.7}$
& $\underline{94.2}_{\pm 0.5}$
& $86.5_{\pm 0.8}$
& $90.3_{\pm 0.1}$
& $\underline{92.0}_{\pm 0.3}$
& $76.9_{\pm 1.7}$
& 83.0 \\

& GPart \textit{(Ours)} & 23K
& \cellcolor{yellow!25}{$\underline{60.6}_{\pm 1.9}$}
& \cellcolor{yellow!25}{$\textbf{94.3}_{\pm 0.5}$}
& \cellcolor{yellow!25}{$\textbf{88.5}_{\pm 0.6}$}
& \cellcolor{yellow!25}{$\underline{90.4}_{\pm 0.1}$}
& $91.1_{\pm 0.1}$
& \cellcolor{yellow!25}{$\textbf{79.8}_{\pm 1.4}$}
& \cellcolor{yellow!25}{\textbf{84.1}} \\

\midrule
\multirow{7}{*}{\rotatebox[origin=c]{90}{RoBERTa-Large}}
& LP & 0
& $44.9_{\pm 1.1}$ & $86.8_{\pm 0.6}$ & $72.3_{\pm 1.1}$
& $58.8_{\pm 0.3}$ & $66.8_{\pm 0.7}$ & $58.8_{\pm 0.6}$ & 64.7 \\

& FF & 355M
& $66.3_{\pm 0.5}$ & $95.8_{\pm 0.3}$ & $89.5_{\pm 0.2}$
& $92.0_{\pm 0.4}$ & $94.6_{\pm 0.2}$ & $83.4_{\pm 1.2}$ & 86.9 \\
\cmidrule(lr){2-10}

& LoRA ($r{=}8$) & 786K
& $\underline{65.3}_{\pm 1.6}$
& $\underline{95.6}_{\pm 0.1}$
& $\underline{88.0}_{\pm 0.6}$
& $91.3_{\pm 0.2}$
& $\textbf{94.8}_{\pm 0.2}$
& $\underline{85.4}_{\pm 0.2}$
& \textbf{86.7} \\

& BitFit & 271K
& $\textbf{65.4}_{\pm 0.6}$
& $\textbf{95.9}_{\pm 0.2}$
& $87.8_{\pm 1.2}$
& $89.9_{\pm 0.4}$
& $94.0_{\pm 0.2}$
& $82.3_{\pm 12.5}$
& 85.9 \\

& VeRA ($r{=}256$) & 61K
& $59.1_{\pm 3.6}$
& $\textbf{95.9}_{\pm 0.2}$
& $87.8_{\pm 0.5}$
& $91.0_{\pm 0.2}$
& $\underline{94.1}_{\pm 0.5}$
& $\textbf{87.2}_{\pm 0.9}$
& 85.8 \\

& Uni-LoRA & 23K
& $65.0_{\pm 7.3}$
& $95.4_{\pm 0.1}$
& $\textbf{88.7}_{\pm 0.5}$
& $\underline{91.5}_{\pm 0.4}$
& $92.9_{\pm 1.0}$
& $81.8_{\pm 3.1}$
& 85.9 \\

& GPart \textit{(Ours)} & 23K
& $64.2_{\pm 0.1}$
& $95.4_{\pm 0.2}$
& $87.2_{\pm 0.8}$
& \cellcolor{yellow!25}{$\textbf{91.7}_{\pm 0.2}$}
& \cellcolor{yellow!25}{$94.0_{\pm 0.2}$}
& \cellcolor{yellow!25}{$85.2_{\pm 0.9}$}
& \cellcolor{yellow!25}{\underline{86.3}} \\

\bottomrule
\end{tabular}%
}
\end{table}

We evaluate GPart on GLUE~\citep{wang2019glue} using RoBERTa-base and RoBERTa-large~\citep{liu2019roberta}. In Table~\ref{tab:roberta-combined}, we report results on CoLA, SST-2, MRPC, STS-B, QNLI, and RTE, using Matthews correlation for CoLA, Pearson correlation for STS-B, and accuracy for the remaining tasks, following standard practice~\citep{wang2019glue,hu2022lora,li2025unilora}. For each task, we reserve a portion of the training set as a development set for checkpoint selection and report the median and standard deviation across three runs on the validation set. For GPart, these runs vary both the training seed and the projection seed, and therefore also sample independent random partitions.
All models and adapters are fine-tuned by us under this train/dev/val evaluation protocol to ensure a consistent comparison across methods. Full training details are provided in Table~\ref{tab:impl-glue}.

We compare against LP, FF, BitFit, LoRA, VeRA, and Uni-LoRA. GPart and Uni-LoRA are matched exactly by subspace dimension $d$; LoRA and VeRA are reported at the closest achievable budget according to their original hyper-parameterizations. GPart performs strongly at very small parameter budgets. On RoBERTa-base, GPart achieves an average score of 84.1 with only 23K trainable parameters, outperforming the matched-budget Uni-LoRA baseline and the larger LoRA and VeRA configurations on average. On RoBERTa-large, GPart also improves over Uni-LoRA at the matched budget, although LoRA attains a slightly higher average while using substantially more trainable parameters. Overall, the GLUE results show that GPart remains competitive under an extremely small parameter budget, with its clearest gains on RoBERTa-base.

\subsection{Mathematical reasoning}

\begin{table}[t]
\caption{Mathematical reasoning performance after fine-tuning on MetaMathQA and evaluation on GSM8K and MATH. We report mean test accuracy over three seeds, with standard deviations shown as subscripts. Bold denotes the best result for each model and benchmark. In the Params column, the left value denotes the number of trainable parameters for Uni-LoRA and GPart, while the right value denotes the number of trainable parameters for rank-1 LoRA.}
\label{tab:math-reasoning}
\centering
\small
\resizebox{0.88\textwidth}{!}{%
\begin{tabular}{lccccccc}
\toprule
& & \multicolumn{3}{c}{GSM8K} & \multicolumn{3}{c}{MATH} \\
\cmidrule(lr){3-5}
\cmidrule(lr){6-8}
Model & Params & Uni-LoRA & GPart & LoRA & Uni-LoRA & GPart & LoRA \\
\midrule
Qwen-2.5-0.5B
& 131K/135K
& $45.49_{\pm 0.39}$
& $\mathbf{46.25}_{\pm 0.32}$
& $45.99_{\pm 0.31}$
& $\mathbf{19.77}_{\pm 0.85}$
& $19.76_{\pm 0.99}$
& $19.03_{\pm 0.25}$ \\

Llama-3.1-8B
& 524K/851K
& $68.83_{\pm 0.56}$
& $\mathbf{69.45}_{\pm 0.13}$
& $68.99_{\pm 1.18}$
& $\mathbf{22.72}_{\pm 0.30}$
& $22.39_{\pm 0.18}$
& $22.53_{\pm 0.31}$ \\

Gemma-7B
& 524K/802K
& $\mathbf{72.37}_{\pm 0.93}$
& $71.68_{\pm 0.88}$
& $70.18_{\pm 0.13}$
& $\mathbf{25.90}_{\pm 0.36}$
& $25.10_{\pm 1.33}$
& $23.90_{\pm 0.35}$ \\

Qwen-2.5-3B
& 524K/460K
& $78.64_{\pm 0.38}$
& $78.65_{\pm 0.48}$
& $\mathbf{78.92}_{\pm 0.41}$
& $38.90_{\pm 1.48}$
& $\mathbf{39.55}_{\pm 1.31}$
& $38.73_{\pm 0.18}$ \\

Qwen-2.5-7B
& 524K/630K
& $83.07_{\pm 1.00}$
& $82.92_{\pm 1.06}$
& $\mathbf{83.83}_{\pm 0.31}$
& $45.39_{\pm 1.28}$
& $\mathbf{47.87}_{\pm 1.80}$
& $44.14_{\pm 0.19}$ \\

\midrule
Average
& --
& $69.68$
& $\mathbf{69.79}$
& $69.58$
& $30.54$
& $\mathbf{30.93}$
& $29.67$ \\

\bottomrule
\end{tabular}%
}
\end{table}








We evaluate GPart on mathematical reasoning using a diverse set of pretrained decoder-only, non-reasoning models spanning multiple scales and architectures: Qwen-2.5-0.5B, Qwen2.5-3B, and Qwen-2.5-7B~\citep{qwen2.5}, Gemma-7B~\citep{gemma_2024}, and Llama-3.1-8B~\citep{grattafiori2024llama}. Following the MetaMath setup~\citep{yu2024metamath}, we fine-tune on MetaMathQA and evaluate on GSM8K~\citep{cobbe2021gsm8k} and MATH~\citep{hendrycksmath2021}, reporting exact-match accuracy on the final answer. Our primary comparison is against Uni-LoRA under exactly matched trainable-parameter budgets; we additionally compare with rank-1 LoRA at comparable parameter budgets.

Table~\ref{tab:math-reasoning} shows that GPart, Uni-LoRA, and LoRA perform similarly across the decoder-only models. Averaged over all five architectures, GPart obtains the highest accuracy on both GSM8K and MATH, with 69.79 and 30.93, compared with 69.68 and 30.54 for Uni-LoRA and 69.58 and 29.67 for LoRA, respectively. The differences are small relative to the observed run-to-run variation, so we interpret these results as evidence of competitiveness rather than superiority. At the model level, each method achieves the best result in some settings, with no method consistently dominating across architectures and benchmarks.

\subsection{Computer vision tasks}

\begin{table}[t]
\caption{Comparison with baseline approaches across seven computer vision datasets using ViT-Base and ViT-Large backbones. Bold indicates the best result among PEFT methods, and highlighted entries mark cases where GPart outperforms Uni-LoRA.}
\label{tab:vision_tasks}
\centering
\small
\resizebox{\textwidth}{!}{%
\begin{tabular}{llccccccccc}
\toprule
Model & Method & \# Params & OxfordPets & CIFAR10 & DTD & EuroSAT & FGVC & RESISC45 & CIFAR100 & Avg. \\
\midrule

\multirow{7}{*}{\rotatebox[origin=c]{90}{ViT-Base}}
& LP & 0
& $90.28_{\pm 0.43}$
& $96.41_{\pm 0.02}$
& $69.77_{\pm 0.67}$
& $88.72_{\pm 0.13}$
& $17.44_{\pm 0.43}$
& $74.22_{\pm 0.10}$
& $84.28_{\pm 0.11}$
& 74.45 \\

& FF & 85.8M
& $93.14_{\pm 0.40}$
& $98.92_{\pm 0.05}$
& $77.68_{\pm 1.21}$
& $99.05_{\pm 0.09}$
& $54.84_{\pm 1.23}$
& $96.13_{\pm 0.13}$
& $92.38_{\pm 0.13}$
& 87.45 \\

\cmidrule(lr){2-11}

& FourierFT & 72K
& $93.21_{\pm 0.26}$
& $98.58_{\pm 0.07}$
& $75.09_{\pm 0.37}$
& $98.29_{\pm 0.04}$
& $27.51_{\pm 0.64}$
& $91.97_{\pm 0.31}$
& $91.20_{\pm 0.14}$
& 82.26 \\

& FourierFT & 239K
& $93.05_{\pm 0.34}$
& $98.69_{\pm 0.06}$
& $77.30_{\pm 0.61}$
& $98.78_{\pm 0.11}$
& $32.44_{\pm 0.99}$
& $94.26_{\pm 0.20}$
& $91.45_{\pm 0.18}$
& 83.71 \\

& LoRA ($r{=}4$) & 184K
& $93.48_{\pm 0.26}$
& $98.76_{\pm 0.11}$
& $75.81_{\pm 1.42}$
& $98.67_{\pm 0.09}$
& $37.95_{\pm 0.46}$
& $93.26_{\pm 0.10}$
& $\textbf{92.19}_{\pm 0.14}$
& 84.30 \\

& Uni-LoRA & 72K
& $\textbf{93.91}_{\pm 0.25}$
& $98.68_{\pm 0.07}$
& $76.60_{\pm 1.92}$
& $\textbf{99.00}_{\pm 0.13}$
& $51.42_{\pm 0.56}$
& $94.13_{\pm 0.27}$
& $91.81_{\pm 0.13}$
& 86.51 \\

& GPart \textit{(Ours)} & 72K
& $93.85_{\pm 0.14}$
& \cellcolor{yellow!25}{$\textbf{98.77}_{\pm 0.05}$}
& \cellcolor{yellow!25}{$\textbf{77.52}_{\pm 1.59}$}
& $\textbf{99.00}_{\pm 0.10}$
& \cellcolor{yellow!25}{$\textbf{56.84}_{\pm 0.17}$}
& \cellcolor{yellow!25}{$\textbf{94.29}_{\pm 0.22}$}
& \cellcolor{yellow!25}{$92.11_{\pm 0.11}$}
& \cellcolor{yellow!25}{\textbf{87.48}} \\

\midrule

\multirow{6}{*}{\rotatebox[origin=c]{90}{ViT-Large}}
& LP & 0
& $91.11_{\pm 0.30}$
& $97.78_{\pm 0.04}$
& $73.33_{\pm 0.26}$
& $92.64_{\pm 0.08}$
& $24.62_{\pm 0.24}$
& $82.02_{\pm 0.11}$
& $84.28_{\pm 0.11}$
& 77.97 \\

& FF & 303.3M
& $94.43_{\pm 0.56}$
& $99.15_{\pm 0.04}$
& $81.79_{\pm 1.01}$
& $99.04_{\pm 0.08}$
& $68.25_{\pm 1.63}$
& $96.43_{\pm 0.07}$
& $93.58_{\pm 0.19}$
& 90.38 \\

\cmidrule(lr){2-11}

& FourierFT & 144K
& $94.46_{\pm 0.28}$
& $99.10_{\pm 0.04}$
& $80.83_{\pm 0.43}$
& $98.65_{\pm 0.09}$
& $39.92_{\pm 0.68}$
& $93.86_{\pm 0.14}$
& $93.31_{\pm 0.09}$
& 85.73 \\

& FourierFT & 480K
& $\textbf{94.84}_{\pm 0.05}$
& $99.08_{\pm 0.05}$
& $\textbf{81.88}_{\pm 0.50}$
& $98.66_{\pm 0.03}$
& $51.28_{\pm 0.66}$
& $95.20_{\pm 0.07}$
& $\textbf{93.37}_{\pm 0.11}$
& 87.76 \\

& Uni-LoRA & 144K
& $94.69_{\pm 0.31}$
& $98.90_{\pm 0.15}$
& $78.89_{\pm 0.86}$
& $98.74_{\pm 0.52}$
& $\textbf{61.72}_{\pm 0.29}$
& $\textbf{95.34}_{\pm 0.09}$
& $92.74_{\pm 0.13}$
& \textbf{88.72} \\

& GPart \textit{(Ours)} & 144K
& $93.86_{\pm 0.36}$
& \cellcolor{yellow!25}{$\textbf{99.11}_{\pm 0.03}$}
& $78.74_{\pm 1.08}$
& \cellcolor{yellow!25}{$\textbf{99.08}_{\pm 0.02}$}
& $61.48_{\pm 0.41}$
& $95.09_{\pm 0.19}$
& $92.59_{\pm 0.15}$
& 88.56 \\

\bottomrule
\end{tabular}%
}
\end{table}

We evaluate GPart on seven computer vision benchmarks to test whether the method transfers beyond language tasks. We follow the protocol used in FourierFT~\citep{gao2024fourierft} and later adopted in Uni-LoRA~\citep{li2025unilora}, using ViT-Base and ViT-Large~\citep{dosovitskiy2020vit} pretrained backbones. We exclude Stanford Cars, which was included in prior work, due to the dataset provenance concerns discussed in the Ethics Statement (Section~\ref{sec:ethics}). To align with prior work, we use matched adapter budgets of 72K trainable parameters for ViT-Base and 144K for ViT-Large. All controlled results are averaged over five random seeds. Additional implementation details can be found in Table~\ref{tab:impl-vision}.

Table~\ref{tab:vision_tasks} shows that GPart achieves the strongest average among the matched-budget PEFT methods on ViT-Base and remains competitive with Uni-LoRA on ViT-Large. On ViT-Base, GPart improves over Uni-LoRA by 0.97 points on average. On ViT-Large, GPart performs within 0.16 points of Uni-LoRA. Across both backbones, GPart remains substantially more parameter-efficient than full fine-tuning while approaching its average accuracy.

\subsection{Additional experiments}

\begin{table}[t]
\caption{Design analyses on RoBERTa-base GLUE at $d=23{,}040$. The Fastfood comparison uses the same intrinsic dimension and training protocol. For partition sensitivity, the training seed is fixed while the projection seed is varied across three runs.\\}
\label{tab:gpart-design}
\centering
\small
\resizebox{.88\textwidth}{!}{%
\begin{tabular}{llccccccc}
\toprule
Analysis & Variant & CoLA & SST-2 & MRPC & STS-B & QNLI & RTE & Avg. \\
\midrule

\multirow{2}{*}{Projection}
& Fastfood
& $58.0_{\pm 1.1}$
& $93.6_{\pm 0.2}$
& $87.2_{\pm 0.2}$
& $90.4_{\pm 0.1}$
& $89.4_{\pm 0.1}$
& $\textbf{80.5}_{\pm 1.6}$
& 83.2 \\

& GPart
& $\textbf{60.6}_{\pm 1.9}$
& $\textbf{94.3}_{\pm 0.5}$
& $\textbf{88.5}_{\pm 0.6}$
& $\textbf{90.4}_{\pm 0.1}$
& $\textbf{91.1}_{\pm 0.1}$
& $79.8_{\pm 1.4}$
& \textbf{84.1} \\

\midrule

\multirow{2}{*}{Sharing}
& Per-layer
& $57.7_{\pm 1.0}$
& $93.6_{\pm 0.2}$
& $87.2_{\pm 1.0}$
& $\textbf{90.7}_{\pm 0.2}$
& $90.8_{\pm 0.1}$
& $77.6_{\pm 1.1}$
& 82.9 \\

& Global
& $\textbf{60.6}_{\pm 1.9}$
& $\textbf{94.3}_{\pm 0.5}$
& $\textbf{88.5}_{\pm 0.6}$
& $90.4_{\pm 0.1}$
& $\textbf{91.1}_{\pm 0.1}$
& $\textbf{79.8}_{\pm 1.4}$
& \textbf{84.1} \\

\midrule

Partition sensitivity
& Fixed training seed
& $60.6_{\pm 0.9}$
& $93.2_{\pm 0.2}$
& $87.7_{\pm 0.7}$
& $90.3_{\pm 0.2}$
& $90.7_{\pm 0.4}$
& $79.5_{\pm 3.0}$
& 83.7 \\

\bottomrule
\end{tabular}%
}
\end{table}
\begin{figure}[t]
    \centering
    \begin{minipage}[t]{0.38\textwidth}
        \centering
        \includegraphics[width=\textwidth]{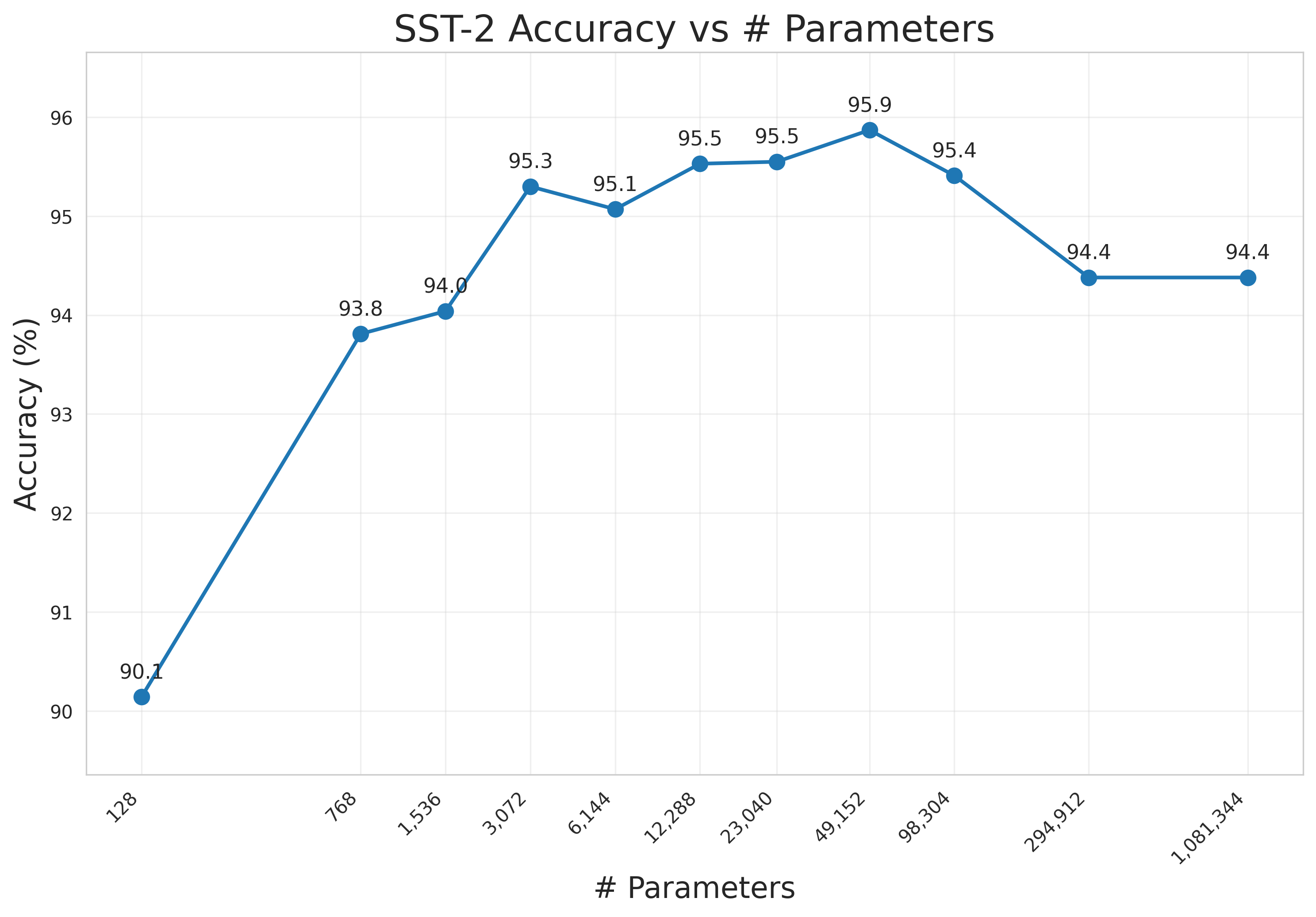}
        \caption{Accuracy on SST-2 with RoBERTa-Large as a function of the subspace dimension $d$.}
        \label{fig:dim-curve-sst2}
    \end{minipage}
    \hfill
    \begin{minipage}[t]{0.58\textwidth}
        \centering
        \includegraphics[width=\textwidth]{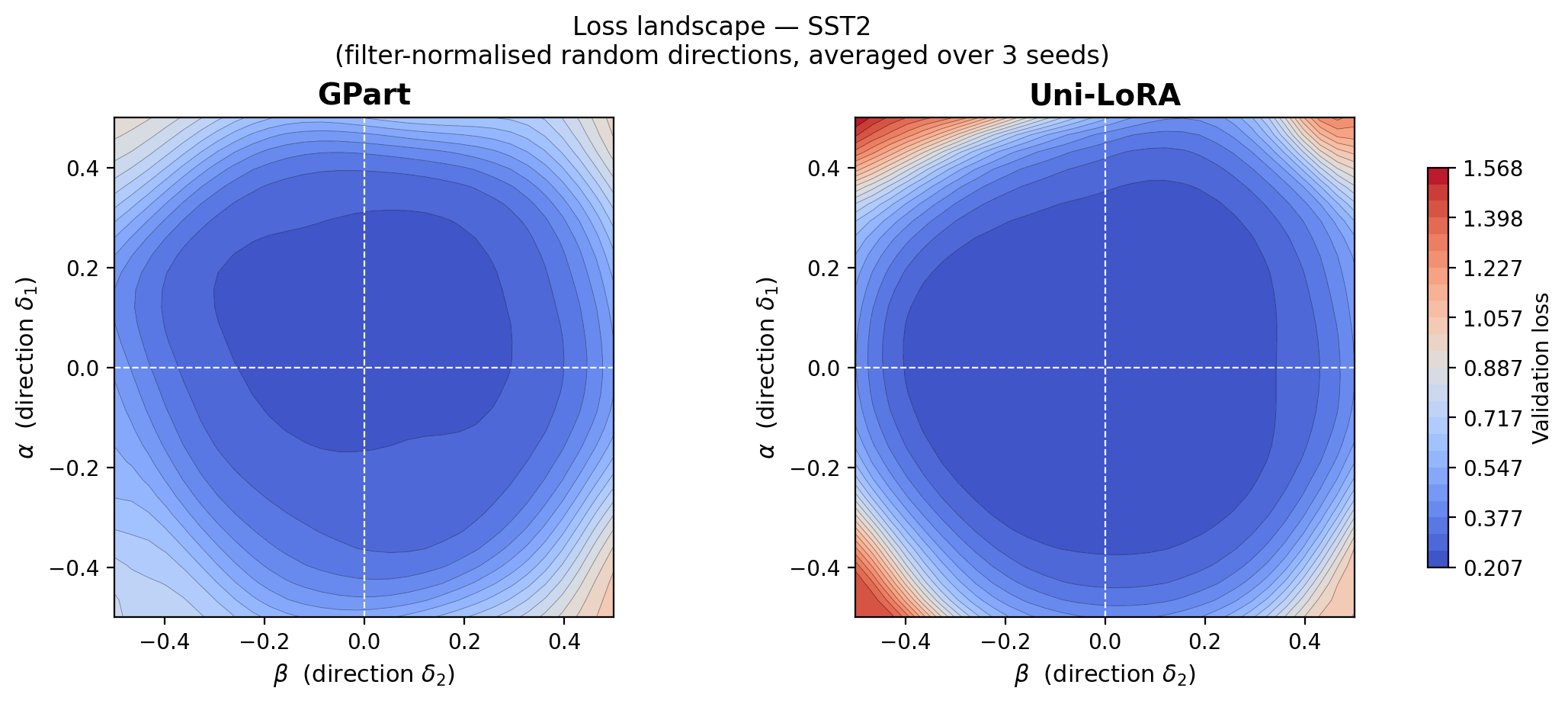}
        \caption{Loss landscape around the converged solution for GPart (left) and Uni-LoRA (right) on SST-2 with RoBERTa-Large and 23K trainable parameters, averaged over three random direction seeds~\citep{visualloss}.}
        \label{fig:loss_landscape}
    \end{minipage}
\end{figure}

We examine GPart's main design choices: subspace dimension and partition sensitivity, global versus per-layer sharing, comparison with Fastfood projection, and optimization geometry.

\textbf{Effect of subspace dimension $d$.} We fine-tune RoBERTa-Large on SST-2 while varying $d$ and keeping all other hyperparameters fixed. Figure~\ref{fig:dim-curve-sst2} shows that performance rises rapidly at small $d$, plateaus in the mid-range, and slightly declines at large $d$. Thus, $d$ provides a direct trade-off between parameter efficiency and model capacity.

\textbf{Sensitivity to the random partition.} We fix the training seed and vary the projection seed across three runs on RoBERTa-base. As shown in Table~\ref{tab:gpart-design}, the standard deviation is at most 0.9 points on five of six GLUE tasks, with higher variability on RTE ($3.0$), suggesting that GPart is generally robust to the sampled partition.

\textbf{Global versus per-layer sharing.} At the same total dimension $d=23{,}040$, global GPart achieves 84.1 average GLUE compared with 82.9 for a per-layer variant and performs better on five of six tasks (Table~\ref{tab:gpart-design}), supporting global sharing in this setting.

\textbf{Comparison with Fastfood.} At the same intrinsic dimension and training protocol, GPart achieves 84.1 average GLUE compared with 83.2 for Fastfood~\citep{aghajanyan2021}, while Fastfood requires approximately $60\%$ more training time (Table~\ref{tab:gpart-design}). GPart is therefore both more accurate and more efficient.

\textbf{Loss landscape geometry.} \label{sec:ablation-landscape} We visualize the optimization geometry of GPart and Uni-LoRA using the filter-normalized random-direction method of \citet{visualloss}, evaluating validation loss on a $30\times30$ grid around the converged solution with perturbations confined to the adapter subspace and averaged over three direction seeds. Figure~\ref{fig:loss_landscape} shows a smooth, well-centered basin for GPart, while Uni-LoRA exhibits sharper and more directionally asymmetric high-loss regions, consistent with the fixed geometry of GPart and the parameter-dependent geometry induced by the bilinear map.

\section{Conclusion}
\label{sec:conclusion}

We introduced GPart, a parameter-efficient fine-tuning method that maps a low-dimensional trainable vector directly into the full model weight space through a random partition matrix. By removing the intermediate low-rank parameterization used by LoRA-based methods, GPart yields an end-to-end isometric parameterization of the optimized subspace while retaining a single hyperparameter, $d$, to control the trainable budget. Across natural language understanding, mathematical reasoning, and computer vision benchmarks, GPart is competitive with existing PEFT methods at ultra-low parameter budgets, with its clearest gains on encoder-only and vision tasks. These results suggest that effective PEFT do not need to rely on a low-rank bottleneck, and that our direct random-subspace parameterization provides a simple and promising alternative.

\textbf{Limitations.} Our evaluation spans encoder-only language models, vision encoders, and decoder-only LLMs, but the decoder-only setting remains limited to a relatively narrow set of models and mathematical reasoning benchmarks. Future work should evaluate GPart on larger language and multimodal models, as well as broader instruction-following and long-context settings.
The current stateless implementation reduces adapter-state and peak memory relative to LoRA and achieves comparable or faster runtime on smaller models. However, on the larger evaluated 7B/8B models, it is currently $2.2\times$--$3.5\times$ slower and does not yet support efficient batched multi-adapter serving. Detailed runtime and memory profiling is provided in Appendix~\ref{app:runtime-memory}.


\subsection*{AI use statement}

In this work, we used generative AI tools for limited assistance with code development and manuscript preparation. Specifically, generative AI was used to assist with small code snippets, debugging, experimental scripting, and deterministic scripts for aggregating results across models and random seeds. It was also used for minor language and presentation improvements, including improving English, refining tables, and checking mathematical expressions.

Generative AI was not used to develop the scientific ideas, hypotheses, methodology, theoretical contributions, or conclusions of this work. The literature review, identification and analysis of related work, data selection and processing, and the substantive scientific content of the manuscript were carried out by the authors without generative AI assistance. The main research implementations, including the proposed adapter, mathematical functions, and evaluation procedures, were developed by the authors or, where AI assistance was used during implementation, were fully supervised and manually inspected by the authors. Figures, diagrams, and other scientific visualizations were created manually.

All AI-assisted text and code were reviewed by the authors. AI-assisted code was manually inspected and tested as part of the experimental pipeline, and scripts used for results collection are deterministic. Citations, literature analysis, and scientific ideas were produced and verified by the authors without generative AI involvement. We take responsibility for the final content of this work, including text, claims, code, and artifacts produced with the aid of generative AI.

\subsection*{Ethics statement}
\label{sec:ethics}

This work does not involve human subjects or the collection of personal or sensitive data. All datasets used in our experiments are publicly available and are cited and described in Section~\ref{sec:experiments} and Appendix~\ref{app:licenses}. The proposed method does not introduce capabilities intended to facilitate harmful use beyond those already present in the underlying pretrained models. We did not identify additional privacy, security, fairness, or discrimination concerns arising specifically from the proposed method.

We also considered the provenance and legal status of datasets when selecting experimental benchmarks. In particular, although Stanford Cars has been used in prior work on this topic, we refrain from using it because of ongoing legal concerns surrounding the collection and distribution of the dataset. Code required to reproduce our experiments is included with the submission and will also be released publicly. We are not aware of additional conflicts of interest or research-ethics concerns specific to this work.

\subsection*{Reproducibility statement}

We provide the source code required to reproduce our experiments as supplementary material with the submission. Our experiments follow standard evaluation protocols adopted in prior work, with training and evaluation procedures described in Section~\ref{sec:experiments} and complete implementation details provided in Appendix~\ref{app:implementation-details}. Pseudocode for the proposed algorithm is provided in Appendix~\ref{app:algorithm}. All datasets and pretrained models used in the experiments are publicly available.

For all reported tables, we report the standard deviation across at least three independent runs with different random seeds, capturing variability due to random initialization and data sampling. We document the computational setup, including the number and type of GPUs, training memory usage, training time, and other relevant implementation details in Section~\ref{sec:experiments} and Appendix~\ref{app:implementation-details}; additional runtime and memory analyses are provided in Appendix~\ref{app:runtime-memory}.

The assumptions and notation underlying our theoretical analysis are stated in Section~\ref{sec:method}. The main theoretical results are presented in Section~\ref{sec:theory}, including the isometry of the GPart projection (Proposition~\ref{prop:isometry}), the non-isometry of LoRA's bilinear map (Proposition~\ref{eq:lora-jacobian}), and the connection to intrinsic dimensionality. Complete proofs and additional theoretical results, including the partition matrix construction, the gradient identity and norm bound, the exact weight-decay correspondence, and the non-isometric ablation theory, are provided in Appendices~\ref{app:proofs} and~\ref{app:noniso-theory}.



\bibliography{references}
\bibliographystyle{iclr2027_conference}

\appendix

\section{Formal Analysis of GPart's Isometric Structure}
\label{app:proofs}

\subsection{Partition matrix construction}

We provide a self-contained construction of the partition matrix and proof of its isometry property.

\begin{definition}[Partition matrix]
Given $N$ parameters, $d$ groups, and a random assignment function $g: \{1,\ldots,N\} \to \{1,\ldots,d\}$, the partition matrix $P \in \R^{N \times d}$ is defined by:
\begin{equation}
P_{ij} = \begin{cases} 1/\sqrt{n_j} & \text{if } g(i) = j, \\ 0 & \text{otherwise}, \end{cases}
\end{equation}
where $n_j = |\{i : g(i) = j\}|$.
\end{definition}

\begin{proposition}
$P^\top P = I_d$.
\end{proposition}
\begin{proof}
$(P^\top P)_{jk} = \sum_{i=1}^N P_{ij} P_{ik}$. If $j \neq k$, the supports of columns $j$ and $k$ are disjoint (each row has a single non-zero entry), so the sum is $0$. If $j = k$, the sum is $\sum_{i: g(i)=j} (1/\sqrt{n_j})^2 = n_j \cdot (1/n_j) = 1$.
\end{proof}

\subsection{Gradient identity}
\begin{proposition}
For any loss function $\mathcal{L}(w)$ with $w = w_0 + P\btheta_d$:
\begin{equation}
\nabla_{\btheta_d} \mathcal{L} = P^\top \nabla_w \mathcal{L}.
\end{equation}
Moreover, $\|\nabla_{\btheta_d} \mathcal{L}\|_2 \leq \|\nabla_w \mathcal{L}\|_2$,
with equality when $\nabla_w \mathcal{L}$ lies entirely in the image of $P$.
\end{proposition}
\begin{proof}
The first identity follows from the chain rule. For the norm bound:
$\|\nabla_{\btheta_d} \mathcal{L}\|^2
= \|P^\top \nabla_w \mathcal{L}\|^2
= (\nabla_w \mathcal{L})^\top P P^\top (\nabla_w \mathcal{L})$.
Since $PP^\top$ is an orthogonal projection onto $\mathrm{image}(P)$,
its eigenvalues are $0$ and $1$, giving the inequality.
\end{proof}


\subsection{Weight decay and regularization}\label{weight_decay}

End-to-end isometry also gives GPart a simple interpretation under L2 regularization. If weight decay is applied to $\btheta_d$, then
\begin{equation}
\lambda \|\btheta_d\|_2^2
=
\lambda \|P\btheta_d\|_2^2
=
\lambda \|\Delta w\|_2^2,
\label{eq:gpart-weight-decay}
\end{equation}
where $\Delta w = P\btheta_d = w - w_0$. Thus, weight decay on the trainable parameters is exactly weight decay on the induced perturbation in full weight space.

For LoRA, regularization is typically applied to the factors $A$ and $B$, yielding the penalty
$$
\lambda \big(\|A\|_F^2 + \|B\|_F^2\big).
$$
This does not directly equal the squared norm of the induced weight perturbation $\Delta W = BA$. Instead, by submultiplicativity of the Frobenius norm and the arithmetic--geometric mean inequality,
\begin{equation}
\|\Delta W\|_F
=
\|BA\|_F
\leq
\|B\|_F \|A\|_F
\leq
\frac{1}{2}\big(\|A\|_F^2 + \|B\|_F^2\big).
\label{eq:lora-regularization-bound}
\end{equation}
Thus, regularization in LoRA controls only an upper bound on the norm of the resulting weight update, rather than the norm itself.

\begin{remark}
For Uni-LoRA, applying weight decay to the trainable vector yields an exact norm interpretation in the intermediate LoRA parameter space, since the projection into that space is isometric. However, after the subsequent bilinear map into weight space, the same mismatch as in LoRA remains.
\end{remark}

\subsection{Relation to structured random projections}

GPart differs from structured random projections such as Fastfood~\citep{le2013fastfood} in the strength of its geometric guarantee. For every realized partition $P$, GPart satisfies
$$
P^\top P = I_d,
$$
and therefore preserves norms and inner products exactly over the entire trainable space.

By contrast, an appropriately normalized Fastfood projection is isotropic in the second-moment sense,
$$
\mathbb{E}[P^\top P] = I_d,
$$
but a particular sampled projection does not generally satisfy $P^\top P=I_d$. Thus, GPart provides an exact per-realization isometry rather than norm preservation only in expectation over the random projection.

We compare the two parameterizations empirically at matched intrinsic dimension and training protocol in Table~\ref{tab:gpart-design}.
\section{The Cost of Breaking Isometry}
\label{app:noniso-theory}

We study the role of the \(1/\sqrt{n_j}\) normalization in GPart by comparing the standard isometric construction with a non-isometric variant in which the normalization is removed. This comparison isolates whether end-to-end isometry is merely a geometric convenience or whether it has a measurable effect on optimization and downstream performance.

\subsection{Setup}

We compare two variants of GPart that differ only in the column normalization of \(P\):
\begin{itemize}
    \item \textbf{GPart (isometric):} \(P_{ij} = 1/\sqrt{n_j}\) if \(g(i)=j\), so \(P^\top P = I_d\).
    \item \textbf{GPart (non-isometric):} \(P_{ij} = 1\) if \(g(i)=j\), so \(P^\top P = \mathrm{diag}(n_1,\ldots,n_d)\).
\end{itemize}
All other aspects of the method are identical.

\subsection{Weight decay miscalibration}

The performance gap between the two variants can be understood as a direct consequence of Equation~\ref{eq:gpart-weight-decay}. In the isometric case, we have
\[
\|\Delta w\|_2^2 = \|P_{\mathrm{iso}} \btheta_d\|_2^2 = \|\btheta_d\|_2^2,
\]
so AdamW weight decay on \(\btheta_d\) with coefficient \(\lambda\) corresponds exactly to \(L_2\) regularization on the induced weight perturbation.

In the non-isometric case, the perturbation associated with group \(j\) has magnitude \(\sqrt{n_j}\,\theta_j\), so
\[
\|\Delta w\|_2^2 = \sum_j n_j \theta_j^2.
\]
AdamW still applies decay directly to \(\theta_j\), so the decay acts at scale \(\|\btheta_d\|_2^2\), while the actual weight perturbation lives at scale \(\sum_j n_j \theta_j^2\). Since \(n_j \approx N/d\) under a roughly balanced random partition, the regularization is miscalibrated by a factor of approximately \(N/d\). For typical settings such as \(N \approx 10^8\) and \(d \approx 10^4\), this factor is on the order of \(10^4\).

This mismatch is not corrected by Adam’s gradient normalization. Equivalently, the non-isometric parameterization behaves like the isometric variant with the induced weight perturbation scaled up by \(\sqrt{n_j}\) while the corresponding regularization strength is scaled down by \(n_j\). In practice, this yields severe under-regularization in weight space.

\subsection{Empirical effect of non-isometric \texorpdfstring{$P$}{P}}
\label{sec:ablation-noniso}

To test whether this theoretical mismatch has practical consequences, we compare GPart against the non-isometric variant on GLUE using RoBERTa-base and RoBERTa-large at a matched budget of 23K trainable parameters. For fairness, we independently tuned the optimization hyperparameters of the non-isometric variant so that any performance difference in performance cannot be attributed to a mismatched optimization scale. The results are reported in Table~\ref{tab:non-isometric}.

\begin{table}[ht]
\caption{Isometric vs.\ non-isometric GPart on GLUE.}
\label{tab:non-isometric}
\centering
\small
\resizebox{.8\textwidth}{!}{%
\begin{tabular}{l lc cccccc c}
\toprule
Model & Method & \# Params & CoLA & SST-2 & MRPC & STS-B & QNLI & RTE & Avg \\
\midrule
\multirow{2}{*}{Base}
& GPart Non-Iso & 23K & $44.7_{\pm 3.4}$ & $93.5_{\pm 0.5}$ & $87.5_{\pm 0.9}$ & $90.3_{\pm 0.2}$ & $\textbf{91.2}_{\pm 0.1}$ & $77.3_{\pm 2.8}$ & 80.7 \\
& GPart & 23K & $\textbf{60.6}_{\pm 1.9}$ & $\textbf{94.3}_{\pm 0.5}$ & $\textbf{88.5}_{\pm 0.6}$ & $\textbf{90.4}_{\pm 0.1}$ & $91.1_{\pm 0.1}$ & $\textbf{79.8}_{\pm 1.4}$ & \textbf{84.1} \\
\midrule
\multirow{2}{*}{Large}
 & GPart Non-Iso & 23K & $60.1_{\pm 0.6}$ & $95.4_{\pm 0.3}$ & $\textbf{89.0}_{\pm 1.6}$ & $89.9_{\pm 2.3}$ & $93.8_{\pm 0.1}$ & $80.5_{\pm 2.2}$ & 84.8 \\
 & GPart & 23K & $\textbf{64.2}_{\pm 0.1}$ & $\textbf{95.4}_{\pm 0.2}$ & $87.2_{\pm 0.8}$ & $\textbf{91.7}_{\pm 0.2}$ & $\textbf{94.0}_{\pm 0.2}$ & $\textbf{85.2}_{\pm 0.9}$ & \textbf{86.3} \\
\bottomrule
\end{tabular}%
}
\end{table}

Across both RoBERTa-base and RoBERTa-large, the isometric variant achieves a higher average score than the non-isometric variant. The gap is especially pronounced on CoLA for the base model (60.6 vs.\ 44.7) and on RTE for the large model (85.2 vs.\ 80.5). Since the partition assignments, parameter-sharing pattern, and subspace dimension are held fixed, this result isolates the practical effect of the isometric normalization.

These results are consistent with the analysis above. Removing the \(1/\sqrt{n_j}\) factor does not merely alter the parameterization algebraically; it changes the relationship between parameter-space regularization and the actual magnitude of the induced weight update. The empirical degradation in Table~\ref{tab:non-isometric} therefore supports the view that the isometric normalization is a practically important part of GPart rather than a cosmetic design choice.
\section{Algorithm}
\label{app:algorithm}

Algorithm~\ref{alg:gpart} summarizes the GPart training procedure. Given a fixed random seed, GPart constructs a deterministic partition of the pretrained parameters into \(d\) groups and optimizes a single scalar per group. The partition and group sizes implicitly define the projection, avoiding the need to materialize a dense projection matrix. During training, the resulting update is added to the frozen pretrained parameters, while optimization is performed only over the \(d\)-dimensional parameter vector \(\btheta_d\). The learned parameters and partition seed are sufficient to reconstruct the adapted model.

\begin{algorithm}[ht]
\caption{GPart: Global Partition Fine-Tuning}
\label{alg:gpart}
\begin{algorithmic}[1]
\REQUIRE Pretrained parameters \(w_0 \in \mathbb{R}^N\), training set \(\mathcal{D}\), subspace dimension \(d\), seed \(s\), optimizer \(\mathrm{Opt}\)
\STATE Generate a fixed partition map \(g : \{1,\ldots,N\} \to \{1,\ldots,d\}\) from seed \(s\)
\STATE Compute group sizes \(n_j = |\{i : g(i)=j\}|\) for \(j=1,\ldots,d\)
\STATE Initialize $\btheta_d = 0$
\STATE Implement \(P\) implicitly using only \(g(\cdot)\) and \(\{n_j\}_{j=1}^d\)
\FOR{each minibatch \(\mathcal{B} \subset \mathcal{D}\)}
    \STATE Form \(\Delta w\) implicitly using \((\Delta w)_i \leftarrow (\btheta_d)_{g(i)} / \sqrt{n_{g(i)}}\) for all \(i\)
    \STATE Set adapted parameters \(w \leftarrow w_0 + \Delta w\)
    \STATE Compute loss \(\mathcal{L}_{\mathcal{B}}(w)\)
    \STATE Compute \(\nabla_{\btheta_d}\mathcal{L}_{\mathcal{B}}\) from \(\nabla_w \mathcal{L}_{\mathcal{B}}\)
    \STATE Update \(\btheta_d \leftarrow \mathrm{Opt}(\btheta_d, \nabla_{\btheta_d}\mathcal{L}_{\mathcal{B}})\)
\ENDFOR
\STATE \textbf{return} \(s, \btheta_d\)
\end{algorithmic}
\end{algorithm}
\section{Comparison of Parameter Counts}
\label{app:params}

Table~\ref{tab:params} summarizes the number of trainable parameters introduced by each method, both per adapted layer and globally. Full fine-tuning scales with the total model size $N$, while LoRA introduces $r(m+n)$ parameters per layer, giving a global count $D = \sum_\ell r(m_\ell + n_\ell)$ that grows with depth and layer width. BitFit and VeRA are more frugal, tuning only bias vectors or diagonal scaling factors, but their global counts still accumulate over layers. FourierFT decouples the budget from layer dimensions by fixing $d$ frequencies per layer, yielding a global count of $L \times d$. GPart and Uni-LoRA go one step further: both parameterize the entire adaptation through a \emph{single} global vector $\boldsymbol{\theta}_d \in \mathbb{R}^d$, shared across all $L$ adapted layers, so the parameter count is independent of model depth and layer dimensions. The key distinction between the two methods is therefore not in parameter count — which is identical by construction in our experiments — but in how $\boldsymbol{\theta}_d$ is mapped back to weight space.

\begin{table}[h]
\caption{Trainable parameter counts per adapted layer and globally across a model with $L$ adapted layers of dimensions $W_\ell \in \mathbb{R}^{m_\ell \times n_\ell}$, with $N$ denoting total model parameters.}
\label{tab:params}
\centering
\begin{tabular}{lcc}
\toprule
Method & Per layer & Global \\
\midrule
Full FT & $mn$ & $N$ \\
LoRA (rank $r$) & $r(m+n)$ & $\sum_\ell r(m_\ell + n_\ell) = D$ \\
BitFit & $m$ (bias only) & $\sum_\ell m_\ell$ \\
VeRA & $m + n$ (diag.\ scaling) & $\sum_\ell (m_\ell + n_\ell)$ \\
FourierFT & $d$ & $L \times d$ \\
Uni-LoRA & \multicolumn{2}{c}{$d$ (global, shared across all layers)} \\
GPart & \multicolumn{2}{c}{$d$ (global, shared across all layers)} \\
\bottomrule
\end{tabular}
\end{table}
\section{Implementation Details}
\label{app:implementation-details}

All experiments are run on a single NVIDIA H100 80GB GPU. We use AdamW throughout, with task-specific learning rates for the classification head and the subspace parameters \(\boldsymbol{\theta}_d\) tuned independently. The partition matrix \(P\) is fixed at initialization and applied to the \texttt{target\_modules} matrices in every transformer block, excluding the task head. Full per-task hyperparameters for natural language understanding, mathematical reasoning, and computer vision are provided in Tables~\ref{tab:impl-glue}, \ref{tab:impl-math}, and~\ref{tab:impl-vision}, respectively.

GPart never explicitly materializes the partition matrix $P$. Instead, the broadcast update is applied implicitly during the forward and backward passes. Each adapted layer gathers the relevant entries of $\boldsymbol{\theta}_d$, rescales and reshapes them into weight updates, and uses a custom autograd function to accumulate gradients directly into the global trainable vector.

Our stateless implementation can additionally derive parameter-to-group assignments directly from the seed and global parameter positions, avoiding the $O(N)$ materialized assignment-index buffer. Detailed runtime and memory comparisons with LoRA, including the resulting memory savings and large-model computational overhead, are provided in Appendix~\ref{app:runtime-memory}.

\begin{table}[ht]
\caption{Hyperparameters for GLUE experiments.}
\label{tab:impl-glue}
\centering
\small
\begin{tabular}{ll cccccc}
\toprule
Model & Hyperparameter & SST-2 & MRPC & CoLA & QNLI & RTE & STS-B \\
\midrule
 & Optimizer           & \multicolumn{6}{c}{AdamW} \\
 & Weight Decay        & \multicolumn{6}{c}{0.1} \\
 & Warmup Ratio        & \multicolumn{6}{c}{0.06} \\
 & LR Schedule         & \multicolumn{6}{c}{Linear} \\
 & Init.\ of $\boldsymbol{\theta}_d$ & \multicolumn{6}{c}{0} \\
 & LR ($\boldsymbol{\theta}_d$) & \multicolumn{6}{c}{5E-3} \\
 & $|\boldsymbol{\theta}_d|$ & \multicolumn{6}{c}{23,040} \\
 & Batch Size Per GPU  & \multicolumn{6}{c}{32} \\
 & Target Modules & \multicolumn{6}{c}{\texttt{["query", "value"]}} \\
 \midrule
\multirow{5}{*}{\rotatebox[origin=c]{90}{\textsc{Base}}}
 & Epochs              & 60    & 30    & 80    & 25    & 160   & 80    \\
 & LR (Head)           & 5E-4  & 1E-3  & 2E-4  & 1E-3  & 1E-2  & 2E-4  \\
 & Max Seq.\ Len.      & \multicolumn{6}{c}{512} \\
 & Training Time (Total)    & \multicolumn{6}{c}{154 minutes} \\
 & GPU Memory    & \multicolumn{6}{c}{19GiB} \\
\midrule
\multirow{5}{*}{\rotatebox[origin=c]{90}{\textsc{Large}}}
 & Epochs              & 30    & 30    & 60    & 25    & 120    & 80    \\
 & LR (Head)           & 2E-4  & 1E-3  & 2E-3  & 1E-3  & 1E-2  & 2E-4  \\
 & Max Seq.\ Len.      & \multicolumn{6}{c}{128} \\
 & Training Time (Total)    & \multicolumn{6}{c}{66 minutes} \\
 & GPU Memory    & \multicolumn{6}{c}{15GiB} \\
\bottomrule
\end{tabular}
\end{table}

\begin{table}[ht]
\caption{Hyperparameters for mathematical reasoning experiments.}
\label{tab:impl-math}
\centering
\small
\resizebox{\textwidth}{!}{%
\begin{tabular}{ll ccccc}
\toprule
 & Hyperparameter & Qwen-2.5-0.5B & Qwen-2.5-3B & Qwen-2.5-7B & Gemma-7B & Llama-3.1-8B \\
\midrule
 & Optimizer            & \multicolumn{5}{c}{AdamW} \\
 & LR Schedule          & \multicolumn{5}{c}{Cosine} \\
 & Batch Size           & \multicolumn{5}{c}{2} \\
 & Accumulation Steps   & \multicolumn{5}{c}{8} \\
 & Warmup Ratio         & \multicolumn{5}{c}{0.05} \\
 & Weight Decay         & \multicolumn{5}{c}{0.05} \\
 & Epochs               & \multicolumn{5}{c}{2} \\
 & Warmup Ratio         & \multicolumn{5}{c}{0.02} \\
 & Target Modules & \multicolumn{5}{c}{\texttt{["q\_proj", "k\_proj", "v\_proj", "o\_proj"]}} \\
 & Init.\ of $\boldsymbol{\theta}_d$ & \multicolumn{5}{c}{0} \\
 & Max Seq. Len.         & \multicolumn{5}{c}{2048} \\
 & LR ($\boldsymbol{\theta}_d$)      & \multicolumn{5}{c}{2E-4}  \\
 & $|\boldsymbol{\theta}_d|$         & 131,072 & 524,288 & 524,288 & 524,288 & 524,288 \\
 & Training Time   & 4h & 8h & 15h & 21h & 17h \\
 & GPU Memory    & 19GiB & 31GiB & 45GiB & 66GiB & 53GiB \\
\bottomrule
\end{tabular}%
}
\end{table}

\begin{table}[ht]
\caption{Hyperparameters for computer vision experiments.}
\label{tab:impl-vision}
\centering
\small
\resizebox{\textwidth}{!}{%
\begin{tabular}{ll ccccccc}
\toprule
Model & Hyperparameter & OxfordPets & CIFAR10 & DTD & EuroSAT & FGVC & RESISC45 & CIFAR100 \\
\midrule
 & Optimizer          & \multicolumn{7}{c}{AdamW} \\
 & Weight Decay       & \multicolumn{7}{c}{0.01} \\
 & LR Schedule        & \multicolumn{7}{c}{Linear} \\
 & Warmup Ratio       & \multicolumn{7}{c}{0.05} \\
 & Init.\ of $\boldsymbol{\theta}_d$ & \multicolumn{7}{c}{0} \\
 & Epochs             & \multicolumn{7}{c}{20} \\
 & Target Modules     & \multicolumn{7}{c}{\texttt{["query", "value"]}} \\
\midrule
\multirow{6}{*}{\rotatebox[origin=c]{90}{\textsc{ViT-Base}}}
 & LR (Head)                    & 2E-2 & 5E-3 & 1E-2 & 1E-2 & 1E-2 & 5E-2 & 5E-3 \\
 & LR ($\boldsymbol{\theta}_d$) & 1E-2 & 1E-2 & 1E-2 & 5E-2 & 5E-2 & 8E-2 & 5E-2 \\
 & $|\boldsymbol{\theta}_d|$    & \multicolumn{7}{c}{72,000} \\
 & Batch Size Per GPU           & \multicolumn{7}{c}{64} \\
 & Training Time                & 4m & 26m & 4m & 11m & 15m & 15m & 27m \\
 & GPU Memory                   & \multicolumn{7}{c}{5GiB} \\
\midrule
\multirow{6}{*}{\rotatebox[origin=c]{90}{\textsc{ViT-Large}}}
 & LR (Head)                    & 2E-2 & 5E-3 & 1E-2 & 1E-2 & 1E-2 & 5E-2 & 5E-3 \\
 & LR ($\boldsymbol{\theta}_d$) & 1E-2 & 1E-2 & 1E-2 & 5E-2 & 5E-2 & 8E-2 & 5E-2 \\
 & $|\boldsymbol{\theta}_d|$    & \multicolumn{7}{c}{144,000} \\
 & Batch Size Per GPU           & \multicolumn{7}{c}{32} \\
 & Training Time                & 6m & 56m & 8m & 24m & 17m & 27m & 51m \\
 & GPU Memory                   & \multicolumn{7}{c}{8GiB} \\
\bottomrule
\end{tabular}%
}
\end{table}
\section{Runtime and Memory Analysis}
\label{app:runtime-memory}

The minimal GPart checkpoint contains only the trainable vector $\boldsymbol{\theta}_d$ and the partition seed. At runtime, however, a materialized implementation additionally requires parameter-to-group indices and normalization scales. With \texttt{int32} indices, the assignment buffer alone requires $4N$ bytes. Our stateless implementation eliminates this buffer by deriving assignments directly from the seed and global parameter positions.

\begin{table}[ht]
\caption{Memory required by materialized assignment indices and the reduction obtained with stateless assignment. Peak memory reports the decrease in peak allocated CUDA memory relative to the materialized GPart implementation.}
\label{tab:stateless-memory}
\centering
\small
\resizebox{\textwidth}{!}{%
\begin{tabular}{lccccc}
\toprule
Metric & Qwen-2.5-0.5B & Qwen-2.5-3B & Qwen-2.5-7B & Gemma-7B & Llama-3.1-8B \\
\midrule
Materialized index buffer & 0.17 GiB & 1.27 GiB & 3.06 GiB & 5.25 GiB & 5.01 GiB \\
Stateless index buffer & 0 GiB & 0 GiB & 0 GiB & 0 GiB & 0 GiB \\
Peak memory reduction & 0.29 GiB & 2.12 GiB & 5.32 GiB & 9.02 GiB & 8.59 GiB \\
\bottomrule
\end{tabular}%
}
\end{table}

For Qwen-2.5-7B, stateless assignment reduces the fixed runtime buffers to 2 MiB and initialization time from 70.79 s to 8.57 s.

We further compare stateless GPart with rank-4 LoRA under identical models, data, precision, batching, and implementation conditions.

\begin{table}[ht]
\caption{Runtime and memory comparison between stateless GPart and rank-4 LoRA. Speed is GPart throughput relative to LoRA; peak $\Delta$ is the reduction in peak allocated CUDA memory for GPart. Adapter state excludes optimizer state.}
\label{tab:runtime-memory}
\centering
\small
\resizebox{.8\textwidth}{!}{%
\begin{tabular}{lccc}
\toprule
Model & GPart speed vs.\ LoRA & Peak $\Delta$ (LoRA$-$GPart) & Adapter state (LoRA/GPart) \\
\midrule
Qwen-2.5-0.5B & $1.10\times$ & +0.73 GiB & 2.06 / 1.00 MiB \\
Qwen-2.5-3B   & $0.93\times$ & +2.40 GiB & 7.03 / 4.00 MiB \\
Qwen-2.5-7B   & $0.45\times$ & +3.47 GiB & 9.62 / 4.00 MiB \\
Gemma-7B      & $0.28\times$ & +3.08 GiB & 12.25 / 4.00 MiB \\
Llama-3.1-8B  & $0.33\times$ & +4.27 GiB & 13.00 / 4.00 MiB \\
\bottomrule
\end{tabular}%
}
\end{table}

Profiles use one NVIDIA H100 80GB GPU, BF16 precision, sequence length 2,048, batch size 2, gradient accumulation of 8, no gradient checkpointing, 10 warmup steps, and 50 measured steps. Peak memory denotes peak allocated CUDA memory.

Stateless GPart uses 43--69\% less fixed adapter-state memory and 0.73--4.27 GiB less peak memory than rank-4 LoRA. Training speed is comparable on the smaller models, but the current implementation is $2.2\times$--$3.5\times$ slower on the evaluated 7B/8B models. Efficient large-model execution and batched multi-adapter serving therefore remain important implementation directions.
\section{Why the \texorpdfstring{$BA$}{BA} Factorization Is Problematic}
\label{app:analysis}

The $BA$ parametrization introduced by \citet{hu2022lora} is a pragmatic choice: it expresses
a rank-$r$ update $\Delta W \in \mathbb{R}^{d_\mathrm{out} \times d_\mathrm{in}}$ using only
$r(d_\mathrm{in} + d_\mathrm{out})$ parameters, a substantial saving when $r \ll \min(d_\mathrm{in},
d_\mathrm{out})$.
Despite its empirical success, this factorization introduces a collection of theoretical
pathologies that have spawned a large body of follow-up work.
We identify the two root causes---a representational non-uniqueness
and the asymmetric coupling of the bilinear factors---and show that
GPart eliminates both by construction.

\subsection{Representational non-uniqueness of LoRA}
\label{sec:gauge}

The bilinear factorization $\Delta W = BA$ is not unique.
For any invertible matrix $G \in \mathbb{R}^{r \times r}$, the substitution
\[
    B \;\longmapsto\; BG^{-1}, \qquad A \;\longmapsto\; GA
\]
leaves $\Delta W$ unchanged.
The set of all pairs that produce a given $\Delta W$ therefore forms an
equivalence class
\[
    [(B, A)] \;=\; \bigl\{(BG^{-1},\, GA) : G \in GL(r)\bigr\},
\]
and the true object of interest---the linear map $\Delta W$---is an element
of the quotient.
Any quantity computed from $(B, A)$ that is not invariant under this family
of substitutions is a property of the \emph{representation}, not of the
underlying adapter.

This non-uniqueness has concrete consequences throughout the LoRA literature:

\paragraph{Routing and merging.}
Methods that operate directly on the raw $(B,A)$ factorization---such as the
rank-one routing of \citet{hilora2025} and the clustering-based merging of
\citet{lego2025}---perform computations on representation-dependent objects.
Two adapters that implement identical linear maps $\Delta W$ but were trained with
different random seeds will in general have different $(B, A)$ pairs, and will
therefore be treated as dissimilar by any metric defined on those pairs.
The permutation invariance claimed by \citet{lego2025} accounts only for the
discrete subgroup of signed permutation matrices inside $GL(r)$; the remaining
continuous non-uniqueness goes unaddressed.

\paragraph{Scale non-invariance.}
A special case of this non-uniqueness is the \emph{scale symmetry}:
$(B, A) \mapsto (\lambda B,\, \lambda^{-1} A)$ for any $\lambda \neq 0$.
Under this transformation $\Delta W$ is unchanged, but any method that computes
norms or distances on the columns of $B$ or the rows of $A$ individually will
produce different results.
This affects, for instance, the token-level routing scores of \citet{hilora2025},
which score hidden states against columns of $B$ and are therefore arbitrarily
sensitive to the optimizer-induced scaling of the factorization.

\paragraph{The SVD as a canonical representative.}
\citet{ostapenko2024towards} and \citet{fleshman2025spectr} avoid this problem
by working with the singular value decomposition $\Delta W = U\Sigma V^\top$
rather than the raw $(B, A)$ pair.
When all singular values are distinct---which holds generically for fully
trained adapters, since the set of matrices with repeated singular values
has measure zero in $\mathbb{R}^{d_\mathrm{out} \times d_\mathrm{in}}$---this
picks a unique representative from each equivalence class, up to a finite
group of sign flips on singular vector pairs.
Since both methods operate post-hoc on trained adapters, this is a principled
and largely complete fix within the LoRA framework.
GPart eliminates the non-uniqueness entirely by construction: the map
$\theta_d \mapsto P\theta_d$ is injective, so every adapter has a unique
representation without requiring a post-hoc canonicalization step.

\subsection{Pathologies of the bilinear factorization}
\label{sec:factorization}

The factorization $\Delta W = BA$ introduces two coupled matrices whose
asymmetric roles generate a cluster of optimization pathologies.
Unlike the representational non-uniqueness discussed in
Section~\ref{sec:gauge}, these pathologies are not about which
$\Delta W$ can be expressed, but about how the optimization
landscape over $(B, A)$ behaves relative to the underlying
objective over $\Delta W$.
GPart eliminates both by replacing the bilinear factorization
with a single linear map $\theta_d \mapsto P\theta_d$, removing
the coupling entirely.

\paragraph{Asymmetric optimization dynamics.}
The gradient of the loss with respect to $B$ depends on $A$ and vice versa,
creating coupled dynamics that are a structural consequence of the bilinear
factorization and are not present in full fine-tuning.
\citet{hayou2024loraplus} showed that using the same learning rate for $A$
and $B$ is provably suboptimal for large-width networks, and proposed LoRA$+$
which assigns different learning rates to the two matrices.
\citet{kalajdzievski2023rslora} showed that the standard $\alpha/r$ scaling
causes gradient collapse as rank increases, and derived the corrected
$\alpha/\sqrt{r}$ scaling.
Both issues are absent in GPart by construction: since there is no bilinear
factorization, there are no coupled factors $A$ and $B$ to which these
pathologies can apply.

\paragraph{Initialization pathology.}
LoRA initializes $B = 0$ so that $\Delta W = BA = 0$ at the start of training,
ensuring the fine-tuned model begins from the pretrained weights.
This convention is structurally forced by the bilinear factorization: there is no
symmetric way to initialize $(B, A)$ to zero output while keeping both matrices
in a generic position.
The asymmetry has spawned a line of work proposing alternative initializations,
including PiSSA \citep{meng2024pissa}, which initializes from the top-$r$ singular
components of $W_0$, and LoftQ \citep{li2023loftq}, which handles the quantized
setting.
Recent work suggests that these gains may be largely attributable to the learning
rate regimes they induce rather than the initializations themselves
\citep{lee2026learningratemattersvanilla}, which is consistent with the view that the initialization
literature is patching a symptom rather than the root cause.
In GPart, any initialization of $\btheta_d$ produces a well-defined $\Delta W = P\btheta_d$
from step one, and the question of how to initialize the adapter does not interact
with the structure of the map itself.

\subsection{Summary}
\label{sec:analysis_summary}

Table~\ref{tab:pathologies} summarizes the pathologies of the $BA$
factorization discussed above, the fixes proposed in the literature,
and whether each is absent in GPart by construction.

\begin{table}[h]
\caption{Pathologies of the LoRA $BA$ factorization and their status
in GPart. $\checkmark$ = present; $\times$ = absent by construction.
The final column indicates structural absence, not a fix.}
\label{tab:pathologies}
\centering
\small
\begin{tabular}{lccc}
\toprule
Pathology & LoRA & Proposed fix & GPart \\
\midrule
$GL(r)$ non-uniqueness    & $\checkmark$ & ARROW, SpectR          & $\times$ \\
Scale non-invariance      & $\checkmark$ & ARROW, SpectR          & $\times$ \\
Coupled gradient dynamics & $\checkmark$ & LoRA$+$, rsLoRA        & $\times$ \\
Initialization pathology  & $\checkmark$ & PiSSA, LoftQ           & $\times$ \\
\bottomrule
\end{tabular}
\end{table}

The pattern is consistent: every identified pathology traces to either
the representational non-uniqueness of the $BA$ factorization or the
asymmetric coupling of its two factors, and GPart eliminates both
simultaneously by replacing the bilinear map with a single linear
projection.
These differences follow structurally from replacing the bilinear $BA$ factorization with a single linear parameterization, reconnecting the adapter coordinates more directly with the random-subspace perspective of \citet{aghajanyan2021}.

\section{Partition Structure and Adapter Merging}
\label{app:partition-merging}

GPart and intrinsic-dimension fine-tuning with Fastfood share the same broad
template: both learn $\theta\in\mathbb{R}^{d}$ and map it directly into an
$N$-dimensional weight update, $\Delta w=P\theta$.  The distinction is
therefore not the use of a random subspace itself, but the structure of its
embedding.  Fastfood implements a structured dense transform and is isometric
in expectation over its random factors.  In contrast, the normalized GPart
partition satisfies $P^\top P=I_d$ for every draw and has exactly one nonzero element
per row.  Thus, if weight $i$ belongs to group $g(i)$ of size $n_{g(i)}$,
\begin{equation}
    (P\theta)_i = \frac{\theta_{g(i)}}{\sqrt{n_{g(i)}}}.
    \label{eq:partition-coordinate-separable}
\end{equation}
This coordinate separability is immaterial for a linear merge, but becomes
relevant for nonlinear, coordinatewise merge rules.  This is not a niche case:
several widely used methods sparsify or select task-vector coordinates before
aggregation.  TIES removes small-magnitude entries and resolves sign conflicts
\citep{NEURIPS2023_1644c9af}; DARE randomly drops and rescales delta parameters
\citep{yu2024language}; Model Breadcrumbs removes both small changes and
magnitude outliers \citep{davari2024model}; DELLA assigns drop probabilities
using magnitude ranks \citep{deep2024dellamergingreducinginterferencemodel};
and MagMax selects, at each coordinate, the task-vector entry with the largest
absolute magnitude \citep{marczak2024magmax}.  A dense projection changes the
coordinates on which all of these masks or decisions are defined.

\paragraph{Commutation with TIES.}
Suppose adapters $\theta_1,\ldots,\theta_m$ share the same projection $P$, let
$\mu$ denote a merge in adapter coordinates, and let $M$ be the corresponding
merge after expansion into weight space.  The two routes commute when
\begin{equation}
    P\mu(\theta_1,\ldots,\theta_m)
    =M(P\theta_1,\ldots,P\theta_m).
    \label{eq:merge-commutation}
\end{equation}
Weighted averaging satisfies Equation~\eqref{eq:merge-commutation} for every
linear $P$, including both GPart and Fastfood, and hence does not distinguish
the two parameterizations.  TIES, however, trims entries by magnitude, elects
a sign coordinatewise, and averages only sign-consistent entries.  These
operations are not invariant to a dense change of basis.

For the partition, all expanded entries in a group have the same sign and
magnitude within each adapter.  The positive factor in
Equation~\eqref{eq:partition-coordinate-separable} therefore preserves sign
election, and the disjoint mean factors through $P$.  Consequently, without
trimming, TIES commutes exactly for arbitrary group sizes.  With top-$\kappa$
trimming, only groups intersecting the trimming boundary can disagree; balanced
groups make this effect small.  Fastfood mixes many entries of $\theta$ into
each weight coordinate, so even sign election alone can produce different
decisions before and after expansion.

\paragraph{Text experiment.}
We test this distinction on the six GLUE tasks used in the main paper with
RoBERTa-base.  Both arms adapt query and value matrices, with
$N=14{,}155{,}776$ and the matched budget $d=23{,}040$.  The GPart arm uses
balanced random groups.  We merge all $\binom{6}{2}=15$ task pairs, pairing
only adapters with the same projection seed.  Table~\ref{tab:glue-ties-geometry}
reports geometry for seed~0 and downstream scores over three seeds
(45 pair--seed units).  The commutation residual is
$\lVert P\mu(\theta)-M(P\theta)\rVert_2/\lVert M(P\theta)\rVert_2$.
The reconstruction error instead merges in weight space and maps back with
$P^\top$; for GPart this is the exact least-squares pullback because
$P^\top P=I_d$.

\begin{table}[t]
    \centering
    \caption{Pairwise TIES merging on GLUE.  Geometry is averaged over 15
    task pairs at seed~0; merged score is averaged over 15 pairs and three
    seeds.  Lower is better for the two error columns and higher is better for
    score.}
    \label{tab:glue-ties-geometry}
    \small
    \begin{tabular}{clrrr}
        \toprule
        $\kappa$ & Projection & Comm. residual & Pullback error & Merged score \\
        \midrule
        0.4 & GPart partition & 0.0090 & 0.0008 & 0.7410 \\
        0.4 & Fastfood        & 0.6902 & 0.4918 & 0.7369 \\
        1.0 & GPart partition & 0.0000 & 0.0000 & 0.7426 \\
        1.0 & Fastfood        & 0.7050 & 0.5020 & 0.7309 \\
        \bottomrule
    \end{tabular}
\end{table}

The geometric separation is large.  At $\kappa=0.4$, the partition is
approximately $77\times$ closer to commuting
than Fastfood and loses less than $10^{-3}$ of the weight-space merge after
pullback.  At $\kappa=1$, the partition commutes up to floating-point error,
whereas Fastfood retains a relative residual of $0.705$; this isolates dense
basis mixing and sign election from the trimming step.

TIES is the measured case study rather than the boundary of the argument.
Deterministic coordinatewise rules that depend only on magnitude and sign act
groupwise on Equation~\eqref{eq:partition-coordinate-separable}; this covers
the masking logic behind Model Breadcrumbs and the magnitude-ranking component
of DELLA, up to threshold ties.  It also makes maximum-magnitude selection, as
used by MagMax, commute exactly: the common positive group scale cannot change
which adapter has the largest absolute value.  Stochastic rules require one
extra condition: DARE or DELLA remains groupwise only when the random mask is
coupled within each partition group.  Independent weight-level masks can split
a group and therefore need not commute.  Thus, GPart supplies a natural
compressed coordinate system for this broader class of sparse or selective
merging methods, while the table directly validates only TIES.

The downstream result is encouraging for the text setting of this paper but
should be interpreted cautiously.  The partition is ahead by $0.41$ points at
$\kappa=0.4$ and by $1.17$ points without trimming, winning 34 of 45 comparisons
in the latter setting.  However, the $\kappa=0.4$ advantage did not pass the
pre-specified interaction test ($p=0.0621$), the density sweep was exploratory,
and a separate vision replication did not show an accuracy advantage for the
partition.  We therefore view merge commutation and compact recoverability as
structural benefits of GPart, not as a guarantee of higher merged accuracy.
Together with exact per-draw isometry and a linear-time gather/scatter
implementation, this is a concrete difference between the partition and a
dimension-matched Fastfood projection that is specific to modular adapters.

Finally, merging in native coordinates requires adapters generated with the same
projection, including the same seed.  Adapters trained with different random
partitions inhabit different coordinate systems and must first be expanded to
weight space.

\paragraph{Merging in weight space vs. adapter space.}
The preceding analysis in Appendix~\ref{app:analysis} identifies several
pathologies arising from the bilinear LoRA factorization
$\Delta W = BA$. We now examine how these pathologies affect downstream
adapter merging. In particular, we compare the spaces in which different
parameterizations perform merging and show that GPart avoids the resulting
mismatch by construction.

For clarity, we first consider additive linear merging; the same argument
applies to weighted linear combinations up to the corresponding scalar
coefficients. Let
\[
\Delta W_1 = B_1A_1,
\qquad
\Delta W_2 = B_2A_2
\]
denote two trained LoRA updates for the same model weight $W_0$. Merging
the induced weight updates gives
\[
\Delta W_1+\Delta W_2
=
B_1A_1+B_2A_2.
\]
In contrast, merging the LoRA factors before multiplication gives
\[
(B_1+B_2)(A_1+A_2)
=
B_1A_1+B_2A_2+B_1A_2+B_2A_1.
\]
Thus, factor-space merging introduces the cross terms
$B_1A_2$ and $B_2A_1$, which are not present in either trained update.
Consequently, merging the factors is generally not equivalent to merging
the corresponding updates in weight space. Moreover, because the
factorization $BA$ is non-unique, factor-space operations can depend on the
particular representation of an update rather than only on the update
itself.

UniLoRA removes the need to directly average independently parameterized
LoRA factors. It represents the complete collection of LoRA parameters as
a vector $\Theta_D\in\mathbb{R}^{D}$ and parameterizes this vector using a
lower-dimensional vector $\Theta_d\in\mathbb{R}^{d}$:
\[
\Theta_D=P_{\mathrm U}\Theta_d,
\qquad
P_{\mathrm U}\in\mathbb{R}^{D\times d}.
\]
After reconstructing $\Theta_D$, UniLoRA reshapes its entries into the
individual matrices $\{B_\ell,A_\ell\}_{\ell=1}^{L}$ and forms the model
updates through
\[
\Delta W_\ell=B_\ell A_\ell.
\]
The reconstruction of $\Theta_D$ from $\Theta_d$ is linear and exact within
the UniLoRA parameterization. However, the subsequent map from the
reconstructed LoRA parameters to model-weight updates remains bilinear.
Therefore, although compact-vector merging is well-defined, it does not
generally coincide with additive merging of the resulting weight updates:
after merging two vectors, the reconstructed factors are added and their
product again contains cross terms. UniLoRA thus resolves the ambiguity of
directly manipulating separately parameterized LoRA factors, but its
compact coordinates still represent LoRA factor space rather than model
weight space.

GPart instead maps its trainable vector directly into model-weight space:
\[
\Delta W=P_{\mathrm G}\Theta_d,
\qquad
P_{\mathrm G}\in\mathbb{R}^{N\times d},
\]
where $P_{\mathrm G}$ is the partition projection defined in
Eq.~\ref{eq:partition-matrix}. Because this map is linear, vector-space
merging is exactly equivalent to weight-space merging:
\[
P_{\mathrm G}
\bigl(\Theta_d^{(1)}+\Theta_d^{(2)}\bigr)
=
P_{\mathrm G}\Theta_d^{(1)}
+
P_{\mathrm G}\Theta_d^{(2)}
=
\Delta W^{(1)}+\Delta W^{(2)}.
\]
Unlike UniLoRA, GPart therefore does not require an intermediate
reconstruction into bilinear LoRA factors. Its merge coordinates are
coordinates of the actual additive model-weight update. This distinction is
particularly important for nonlinear coordinate-wise merging rules such as
TIES: GPart applies trimming and sign selection directly to a structured
weight-space representation, whereas UniLoRA applies them to a compressed
representation of the LoRA factors.

\paragraph{Merging experiments on GLUE.}
We evaluate the preceding analysis on the six GLUE tasks used in the main
paper, using both RoBERTa-base and RoBERTa-large. We merge all
$\binom{6}{2}=15$ task pairs and, for methods with a shared random
projection, pair only adapters generated with the same projection seed.
LoRA adapters are merged in factor space, whereas UniLoRA and GPart
adapters are merged in their respective compact vector spaces.

For comparison, we also evaluate two weight-space LoRA baselines. In
\textsc{LoRA-Direct-W}, each LoRA adapter is first converted into its
weight update $\Delta W=BA$, after which the updates are merged directly in
weight space. In \textsc{LoRA-W-SVD}, the merged weight update is further
compressed back into rank-$r$ LoRA factors using an SVD. The
difference between these two baselines isolates the effect of rank-constrained refactorization after merging.

\begin{table}[t]
\centering
\caption{Average relative merging degradation (\%) on GLUE tasks for RoBERTa-base and RoBERTa-large. Results are averaged over 15 task pairs and reported as mean $\pm$ standard deviation across three independent seeds. Lower is better. \textbf{Bold} denotes the best overall result per merging algorithm; \underline{underline} denotes the best result within adapter-space merges; $^\star$ denotes adapter space merging; $^\dagger$ denotes full weight space merging.}
\label{tab:roberta_relative_degradation}
\begin{tabular}{llcc}
\toprule
Model & Method & Linear & TIES ($k=0.4$) \\
\midrule
\multirow{5}{*}{RoBERTa-base}
& LoRA$^\star$          & $23.15 \pm 4.15$ & $15.81 \pm 1.14$ \\
& UniLoRA$^\star$          & $22.48 \pm 3.49$ & $24.51 \pm 2.40$ \\
& GPart$^\star$            & $\underline{19.22 \pm 1.41}$ & $\underline{13.03 \pm 0.39}$ \\
& LoRA-Direct-W$^\dagger$ & $\mathbf{10.55 \pm 2.10 }$ & $6.97 \pm 0.83$ \\
& LoRA-W-SVD$^\dagger$    & $10.78 \pm 2.04$ & $\mathbf{6.96 \pm 0.92}$ \\
\midrule
\multirow{5}{*}{RoBERTa-large}
& LoRA$^\star$             & $26.23 \pm 2.95$ & $16.40 \pm 1.88$ \\
& UniLoRA$^\star$          & $40.71 \pm 7.30$ & $35.59 \pm 8.01$ \\
& GPart$^\star$            & $\underline{16.41 \pm 1.92}$ & $\underline{10.34 \pm 0.54}$ \\
& LoRA-Direct-W$^\dagger$ & $\mathbf{9.20 \pm 2.19}$ & $\mathbf{7.18 \pm 0.91}$ \\
& LoRA-W-SVD$^\dagger$    & $9.74 \pm 2.26$ & $7.72 \pm 1.09$ \\
\bottomrule
\end{tabular}
\end{table}

GPart achieves the lowest degradation among all adapter-space methods for
both model sizes and both merging rules. The advantage is not limited to linear merging, GPart TIES results suggest that its structured coordinates remain effective when the merge rule performs coordinate-wise trimming and sign selection.

The weight-space LoRA baselines obtain still lower degradation. In
particular, \textsc{LoRA-Direct-W} reduces linear-merging degradation for RoBERTa-base and for RoBERTa-large.
This comparison is important because it shows that the poor performance of
factor-space LoRA merging should not be attributed to the learned LoRA
updates themselves. Once the same updates are converted to $\Delta W$ and
merged additively, their degradation decreases substantially.

The small gap between \textsc{LoRA-Direct-W} and \textsc{LoRA-W-SVD}
suggests that rank-$r$ refactorization is not the dominant source of error
in this setting. For example, the two methods differ by only $0.23$ and
$0.54$ percentage points under linear merging for RoBERTa-base and
RoBERTa-large, respectively. The much larger gap between factor-space LoRA
and \textsc{LoRA-Direct-W} instead implicates the factor-space merge itself,
consistent with the cross terms introduced by
$(B_1+B_2)(A_1+A_2)$.

GPart does not outperform direct full-weight merging. Rather, its
contribution is to approach the behavior of weight-space merging while
retaining a compact adapter representation and avoiding the materialization
of full weight updates during merging. Its merge coordinates are directly
related to the additive model-weight update, so linear merging is exactly
preserved and coordinate-wise rules such as TIES operate on structurally
meaningful groups.

The comparison with UniLoRA is especially revealing. UniLoRA has the
largest degradation in three of the four adapter-space settings and is
particularly weak for RoBERTa-large. This
result is consistent with the distinction between the two parameterizations:
UniLoRA reconstructs the full LoRA factor vector exactly, but the resulting
model update is still formed through the bilinear map $BA$. Consequently,
merging compact UniLoRA vectors does not generally correspond to additive
merging of the induced weight updates. GPart avoids this additional
bilinear stage by mapping its compact vector directly into weight space.

\section{Licenses and asset usage}
\label{app:licenses}

We document the external models and datasets used in this work, together with their source URLs and publicly stated licenses.

\paragraph{Natural language understanding.}
We use RoBERTa-base and RoBERTa-large, released by Facebook AI under the MIT License and available at \url{https://huggingface.co/roberta-base} and \url{https://huggingface.co/roberta-large}. We evaluate on the GLUE benchmark, available at \url{https://gluebenchmark.com/}; GLUE is a collection of public benchmark tasks whose component datasets are distributed under their respective licenses.

\paragraph{Mathematical reasoning.}
We use Qwen2.5-0.5B, Qwen2.5-3B, and Qwen2.5-7B from Qwen, available at \url{https://huggingface.co/Qwen/Qwen2.5-0.5B}, \url{https://huggingface.co/Qwen/Qwen2.5-3B}, and \url{https://huggingface.co/Qwen/Qwen2.5-7B}, released under Apache 2.0 license. We also use Gemma-7B, which is gated under Google’s Gemma usage license at \url{https://huggingface.co/google/gemma-7b}, and Llama-3.1-8B, which is distributed under the Llama 3.1 Community License at \url{https://huggingface.co/meta-llama/Llama-3.1-8B}. For training and evaluation, we use MetaMathQA (\url{https://huggingface.co/datasets/meta-math/MetaMathQA}), GSM8K (\url{https://huggingface.co/datasets/openai/gsm8k}), and MATH (\url{https://github.com/hendrycks/math}); these assets are commonly distributed under permissive academic/open licenses, with MetaMathQA and GSM8K publicly listed under MIT in prior work.

\paragraph{Vision.}
We use the Vision Transformer models ViT-Base and ViT-Large from Google, available at \url{https://huggingface.co/google/vit-base-patch16-224} and \url{https://huggingface.co/google/vit-large-patch16-224} and listed under the Apache 2.0 license. We evaluate on Oxford-IIIT Pets (\url{https://www.robots.ox.ac.uk/~vgg/data/pets/}), CIFAR-10 (\url{https://www.cs.toronto.edu/~kriz/cifar.html}), DTD (\url{https://www.robots.ox.ac.uk/~vgg/data/dtd/}), EuroSAT (\url{https://github.com/phelber/EuroSAT}), FGVC-Aircraft (\url{https://www.robots.ox.ac.uk/~vgg/data/fgvc-aircraft/}), RESISC45 (\url{https://huggingface.co/datasets/timm/resisc45}), and CIFAR-100 (\url{https://www.cs.toronto.edu/~kriz/cifar.html}).

All assets were used in accordance with their publicly stated terms. We did not use proprietary or closed-access datasets.

\end{document}